\documentclass[12pt]{article}
\usepackage{amsmath}
\usepackage{times}
\usepackage{color}
\usepackage{multirow}
\usepackage[authoryear, round]{natbib}
\usepackage{rotating}
\usepackage{bbm}
\usepackage{latexsym}
\usepackage{url}
\usepackage{epstopdf}
\usepackage{amssymb,amsfonts,amsmath,amsthm,bm}
\usepackage{algorithm2e,framed,algorithmic}

\def\math#1{$#1$}

\def\v#1{{\mathbf #1}}
\def\frac#1#2{{#1\over #2}}

\def\x{{\mathbf x}}
\def\y{{\mathbf y}}
\def\z{{\mathbf z}}

\def\a{{\mathbf a}}
\def\b{{\mathbf b}}

\def\dotfil{\leaders\hbox to 1.5mm{.}\hfill}
\newcounter{rmnum}
\def\RN#1{\setcounter{rmnum}{#1}\uppercase\expandafter{\romannumeral\value{rmnum}}}
\def\rn#1{\setcounter{rmnum}{#1}\expandafter{\romannumeral\value{rmnum}}}

\newcommand{\Trace }[1]{\mbox{}{\bf{Tr}}\left(#1\right)}

\newcommand{\TNorm }[1]{\mbox{}\left\|#1\right\|_2  }
\newcommand{\TNormS}[1]{\mbox{}\left\|#1\right\|_2^2}

\newcommand{\setlinespacing}[1]%
           {\setlength{\baselineskip}{#1 \defbaselineskip}}

\newcommand{\abs }[1]{\left|#1\right|}

\newtheorem{lemma}{Lemma}
\newtheorem{theorem}{Theorem}

\newcommand{\mat}[1]{{\ensuremath{\bm{\mathrm{#1}}}}}

\def\betavec{{\bm \beta}}
\def\omegavec{{\bm \omega}}

\def\a{{\bm \alpha}}

\def\b{{\mathbf b}}
\def\e{{\mathbf e}}

\def\q{{\mathbf q}}

\def\x{{\mathbf x}}
\def\u{{\mathbf u}}
\def\v{{\mathbf v}}

\def\matA{\mat{A}}
\def\matB{\mat{B}}

\def\matD{\mat{D}}
\def\matE{\mat{E}}

\def\matI{\mat{I}}
\def\matK{\mat{K}}
\def\matKt{\tilde{\matK}}

\def\matP{\mat{P}}
\def\matQ{\mat{Q}}
\def\matR{\mat{R}}

\def\matS{\mat{S}}

\def\matU{\mat{U}}
\def\matV{\mat{V}}

\def\matX{\mat{X}}

\def\matSig{\mat{\Sigma}}
\def\matDelta{\mat{\Delta}}
\def\matPhi{\mat{\Phi}}

\def\matPsi{\mat{\Psi}}

\newcommand{\captionfonts}{\normalsize}

\makeatletter  
\long\def\@makecaption#1#2{%
  \vskip\abovecaptionskip
  \sbox\@tempboxa{{\captionfonts #1: #2}}%
  \ifdim \wd\@tempboxa >\hsize
    {\captionfonts #1: #2\par}
  \else
    \hbox to\hsize{\hfil\box\@tempboxa\hfil}%
  \fi
  \vskip\belowcaptionskip}
\makeatother   

\begin{document}

\hspace{13.9cm}1

\ \vspace{20mm}\\
{\LARGE Feature Selection for Ridge Regression with Provable Guarantees}

\ \\
{\bf \large Saurabh Paul} \\
{Global Risk Sciences, Paypal Inc., San Jose, CA}\\
{saurabhpaul2006@gmail.com} \\
{\bf \large Petros Drineas}\\
{Computer Science Dept., Rensselaer Polytechnic Institute, Troy, NY.}\\
{drinep@cs.rpi.edu} \\


\thispagestyle{empty}
\markboth{}{NC instructions}
\ \vspace{-0mm}\\

\begin{center} {\bf Abstract}\end{center}
We introduce single-set spectral sparsification as a deterministic sampling based feature selection technique for regularized least squares classification, which is the classification analogue to ridge regression. The method is unsupervised and gives worst-case guarantees of the generalization power of the classification function after feature selection with respect to the classification function obtained using all features. We also introduce leverage-score sampling as an unsupervised randomized feature selection method for ridge regression. We provide risk bounds for both single-set spectral sparsification and leverage-score sampling on ridge regression in the fixed design setting and show that the risk in the sampled space is comparable to the risk in the full-feature space. We perform experiments on synthetic and real-world datasets, namely a subset of TechTC-300 datasets, to support our theory. Experimental results indicate that the proposed methods perform better than the existing feature selection methods.

\section{Introduction}
Ridge regression is a popular technique in machine learning and statistics. It is a commonly used penalized regression method.
Regularized Least Squares Classifier (RLSC) is a simple classifier based on least squares and has a long history in machine learning \citep{PZhang04,Poggio03,RifkinRLSC,fung01,Suykens99,TZhang01,Agarwal02}. RLSC is also the classification analogue to ridge regression. RLSC has been known to perform comparably to the popular Support Vector Machines (SVM) \citep{RifkinRLSC,fung01,Suykens99,TZhang01}. RLSC can be solved by simple vector space operations and do not require quadratic optimization techniques like SVM. \\
We propose a deterministic feature selection technique for RLSC with provable guarantees. There exist numerous feature selection techniques, which work well empirically. There also exist randomized feature selection methods like leverage-score sampling, \citep{Dasgup07} with provable guarantees which work well empirically. But the randomized methods have a failure probability and have to be re-run multiple times to get accurate results. Also, a randomized algorithm may not select the same features in different runs. A deterministic algorithm will select the same features irrespective of how many times it is run. This becomes important in many applications. Unsupervised feature selection involves selecting features oblivious to the class or labels. \\
In this work, we present a new provably accurate unsupervised feature selection technique for RLSC. We study a deterministic sampling based feature selection strategy for RLSC with provable non-trivial worst-case performance bounds. \\
We also use single-set spectral sparsification and leverage-score sampling as unsupervised feature selection algorithms for ridge regression in the fixed design setting. Since the methods are unsupervised, it will ensure that the methods work well in the fixed design setting, where the target variables have an additive homoskedastic noise. The algorithms sample a subset of the features from the original data matrix and then perform regression task on the reduced dimension matrix. We provide risk bounds for the feature selection algorithms on ridge regression in the fixed design setting.\\
The number of features selected by both algorithms is proportional to the rank of the training set. The deterministic sampling-based feature selection algorithm performs better in practice when compared to existing methods of feature selection.
\section{Our Contributions}\vskip -0.1cm
We introduce single-set spectral sparsification as a provably accurate deterministic feature selection technique for RLSC in an unsupervised setting. The number of features selected by the algorithm is independent of the number of features, but depends on the number of data-points. The algorithm selects a small number of features and solves the classification problem using those features. \citet{Dasgup07} used a leverage-score based randomized feature selection technique for RLSC and provided worst case guarantees of the approximate classifier function to that using all features. We use a deterministic algorithm to provide worst-case generalization error guarantees. The deterministic algorithm does not come with a failure probability and the number of features required by the deterministic algorithm is lesser than that required by the randomized algorithm. The leverage-score based algorithm has a sampling complexity of
$O\left(\frac{n}{\epsilon^2} \log\left(\frac{n}{\epsilon^2 \sqrt{\delta}}\right) \right)$, whereas single-set spectral sparsification requires $O\left( n/ \epsilon^2 \right)$ to be picked, where $n$ is the number of training points, $\delta \in (0,1)$ is a failure probability and $\epsilon \in (0,1/2]$ is an accuracy parameter. Like in \citet{Dasgup07}, we also provide additive-error approximation guarantees for any test-point and relative-error approximation guarantees for test-points that satisfy some conditions with respect to the training set. \\
We introduce single-set spectral sparsification and leverage-score sampling as unsupervised feature selection algorithms for ridge regression and provide risk bounds for the subsampled problems in the fixed design setting. The risk in the sampled space is comparable to the risk in the full-feature space. We give relative-error guarantees of the risk for both feature selection methods in the fixed design setting.\\
From an \textbf{empirical perspective}, we evaluate single-set spectral sparsification on synthetic data and 48 document-term matrices, which are a subset of the TechTC-300 \citep{David04} dataset. We compare the single-set spectral sparsification algorithm with leverage-score sampling, information gain, rank-revealing QR factorization (RRQR) and random feature selection. We do not report running times because feature selection is an offline task. The experimental results indicate that single-set spectral sparsification out-performs all the methods in terms of out-of-sample error for all 48 TechTC-300 datasets. We observe that a much smaller number of features is required by the deterministic algorithm to achieve good performance when compared to leverage-score sampling.

\section{Background and Related Work}
\subsection{Notation}
$\matA, \matB, \ldots$ denote matrices and $\a, \b, \ldots$ denote column vectors; $\e_i$ (for all $i=1\ldots n$) is the standard basis, whose dimensionality will be clear from context; and $\matI_n$ is the  $n \times n$ identity matrix. The Singular Value Decomposition (SVD) of a matrix $\matA \in \mathbb{R}^{n \times d}$ is equal to %
$ \matA = \matU \matSig \matV^T,$
where $\matU \in \mathbb{R}^{n \times d}$ is an orthogonal matrix containing the left singular vectors, $\matSig \in \mathbb{R}^{d \times d}$ is a diagonal matrix containing the singular values $\sigma_1 \geq \sigma_2  \geq \ldots \sigma_{d} > 0$, and $\matV \in \mathbb{R}^{d \times d}$ is a matrix containing the right singular vectors.
The spectral norm of $\matA$ is  $\TNorm{\matA} = \sigma_1$. $\sigma_{max}$ and $\sigma_{min}$ are the largest and smallest singular values of $\matA$.
$\kappa_{\matA} = \sigma_{max}/\sigma_{min}$ is the condition number of $\matA$. $\matU^{\perp}$ denotes any $n \times \left(n-d\right)$ orthogonal matrix whose columns span the subspace orthogonal to $\matU$. A vector $\q \in \mathbb{R}^n$ can be expressed as: $\q = \matA\bm{\alpha} + \matU^{\perp}\bm{\beta},$ for some vectors $\bm{\alpha} \in \mathbb{R}^d$ and $\bm{\beta} \in \mathbb{R}^{n-d}$, i.e. $\q$ has one component along $\matA$ and another component orthogonal to $\matA$.

\subsection{Matrix Sampling Formalism}
We now present the tools of feature selection. Let $\matA \in \mathbb{R}^{d \times n}$ be the data matrix consisting of $n$ points and $d$ dimensions, $\matS \in \mathbb{R}^{r\times d}$ be a matrix such that $\matS\matA \in \mathbb{R}^{r\times n}$ contains $r$ rows of $\matA.$ Matrix $\matS$ is a binary $(0/1)$ indicator matrix, which has exactly one non-zero element in each row. The non-zero element of $\matS$ indicates which row of $\matA$ will be selected. Let $\matD \in \mathbb{R}^{r\times r}$ be the diagonal matrix such that $\matD\matS\matA \in \mathbb{R}^{r\times n}$ rescales the rows of $\matA$ that are in $\matS\matA.$ The matrices $\matS$ and $\matD$ are called the sampling and re-scaling matrices respectively. We will replace the sampling and re-scaling matrices by a single matrix $\matR \in \mathbb{R}^{r\times d}$, where $\matR = \matD\matS$ denotes the matrix specifying which of the $r$ rows of $\matA$ are to be sampled and how they are to be rescaled. 

\subsection{RLSC Basics}
Consider a training data of $n$ points in $d$ dimensions with respective labels $y_i \in \{-1,+1\}$ for $i=1,..,n.$ The solution of binary classification problems via Tikhonov regularization in a Reproducing Kernel Hilbert Space (RKHS) using the squared loss function results in Regularized Least Squares Classification (RLSC) problem~\citep{RifkinRLSC}, which can be stated as:
\begin{equation} 
\min_{\x \in \mathbb{R}^n} \TNormS{\matK\x -\y} + \lambda\x^T\matK\x
\label{eqn:rlsc1}
\end{equation}
\noindent where $\matK$ is the $n\times n$ kernel matrix defined over the training dataset, $\lambda$ is a regularization parameter and $\y$ is the $n$ dimensional $\{\pm1 \}$ class label vector. In matrix notation, the training data-set $\matX$ is a $d\times n$ matrix, consisting of $n$ data-points and $d$ features $(d\gg n)$. Throughout this study, we assume that $\matX$ is a full-rank matrix. We shall consider the linear kernel, which can be written as $\matK = \matX^T\matX.$ Using the SVD of $\matX$, the optimal solution of Eqn.~\ref{eqn:rlsc1} in the full-dimensional space is 
\begin{equation}
\x_{opt} = \matV\left(\matSig^2 + \lambda\matI\right)^{-1} \matV^T\y.
\label{eqn:xopt}
\end{equation} 
\noindent The vector $\x_{opt}$ can be used as a classification function that generalizes to test data. If $\q \in \mathbb{R}^d$ is the new test point, then the binary classification function is:
\begin{equation}
f(\q) =\x_{opt}^T\matX^T\q.
\label{eqn:clsf}
\end{equation}
\noindent Then, $sign(f(\q))$ gives the predicted label ($-1$ or $+1$) to be assigned to the new test point $\q$.

Our goal is to study how RLSC performs when the deterministic sampling based feature selection algorithm is used to select features in an unsupervised setting. Let $\matR \in \mathbb{R}^{r \times d}$ be the matrix that samples and re-scales $r$ rows of $\matX$ thus reducing the dimensionality of the training set from \math{d} to  $r \ll d$ and $r$ is proportional to the rank of the input matrix. The transformed dataset into \math{r} dimensions is given by \math{\tilde\matX=\matR\matX} and the RLSC problem becomes
\begin{equation}
\min_{\x \in \mathbb{R}^n} \TNormS{\tilde{\matK}\x -\y} + \lambda\x^T\tilde{\matK}\x,
\label{eqn:rlsc2}
\end{equation}
thus giving an optimal vector $\tilde{\x}_{opt}$. The new test point $\q$ is first dimensionally reduced to $\tilde{\q}=\matR\q$, where $\tilde{\q} \in \mathbb{R}^r$ and then classified by the function, 
\begin{equation}
\tilde{f}= f(\tilde{\q}) =\tilde{\x}_{opt}^T\tilde{\matX}^T\tilde{\q}.
\label{eqn:clsf2}
\end{equation}
In subsequent sections, we will assume that the test-point $\q$ is of the form  $\q = \matX\bm{\alpha} + \matU^{\perp}\bm{\beta}.$ The first part of the expression shows the portion of the test-point that is similar to the training-set and the second part shows how much the test-point is novel compared to the training set, i.e. $\TNorm{\bm\beta}$ measures how much of $\q$ lies outside the subspace spanned by the training set.
\subsection{Ridge Regression Basics}
Consider a data-set $\matX$ of $n$ points in $d$ dimensions with $d\gg n$. Here $\matX$ contains $n$ i.i.d samples from the $d$ dimensional independent variable. $\y \in \mathbb{R}^{n}$ is the real-valued response vector. Ridge Regression(RR) or Tikhonov regularization penalizes the $\ell_2$ norm of a parameter vector $\betavec$ and shrinks the estimated coefficients towards zero. In the fixed design setting, we have $\y =\matX^T\betavec+\omegavec$ where $\omegavec \in \mathbb{R}^n$ is the homoskedastic noise vector with mean 0 and variance $\sigma^2$. Let $\betavec_\lambda$ be the solution to the ridge regression problem. The RR problem is stated as:
\begin{equation} 
\hat{\betavec}_\lambda = \arg \min_{\betavec \in \mathbb{R}^d} \frac{1}{n}\TNormS{\y -\matX^T\betavec} + \lambda\TNormS{\betavec}.
\label{eqn:rr1}
\end{equation}
\noindent The solution to Eqn.\ref{eqn:rr1} is $\hat{\betavec}_\lambda = \left( \matX \matX^T+n\lambda\matI_d\right)^{-1}\matX\y$. One can also solve the same problem in the dual space. Using change of variables, $\betavec = \matX \a$, where $\a \in \mathbb{R}^{n}$  and let $\matK = \matX^T\matX$ be the $n\times n$ linear kernel defined over the training dataset. The optimization problem becomes:
\begin{equation} 
\hat\a_\lambda = \arg \min_{\a \in \mathbb{R}^n} \frac{1}{n}\TNormS{\y -\matK\a} + \lambda\a^T\matK\a.
\label{eqn:rr2}
\end{equation}
Throughout this study, we assume that $\matX$ is a full-rank matrix. Using the SVD of $\matX$, the optimal solution in the dual space (Eqn.~\ref{eqn:rr2}) for the full-dimensional data is given by $\hat\a_\lambda = \left(\matK + n\lambda\matI_n\right)^{-1}\y.$ The primal solution is $\hat\betavec_\lambda =\matX\hat\a_\lambda.$

In the sampled space, we have $\tilde{\matK}=\tilde{\matX}^T\tilde{\matX}.$  The dual problem in the sampled space can be posed as:
\begin{equation} 
\tilde\a_\lambda = \arg \min_{\a \in \mathbb{R}^n} \frac{1}{n}\TNormS{\y -\tilde{\matK}\a} + \lambda\a^T\tilde{\matK}\a.
\label{eqn:rr3}
\end{equation}
The optimal dual solution in the sampled space is $\tilde\a_\lambda = \left(\tilde{\matK} + n\lambda\matI_n\right)^{-1}\y.$
The primal solution is $\tilde\betavec_\lambda =\tilde{\matX}\tilde\a_\lambda.$

\subsection{Related Work} 
The work most closely related to ours is that of \citet{Dasgup07} who used a leverage-score based randomized feature selection technique for RLSC and provided worst case bounds of the approximate classifier with that of the classifier for all features. The proof of their main quality-of-approximation results provided an intuition of the circumstances when their feature selection method will work well. The running time of leverage-score based sampling is dominated by the time to compute SVD of the training set i.e. $O\left( n^2 d\right)$, whereas, for single-set spectral sparsification, it is $O\left(rdn^2\right)$. Single-set spectral sparsification is a slower and more accurate method than leverage-score sampling. Another work on dimensionality reduction of RLSC is that of \citet{Avron13} who used efficient randomized-algorithms for solving RLSC, in settings where the design matrix has a Vandermonde structure. However, this technique is different from ours, since their work is focused on dimensionality reduction using linear combinations of features, but not on actual feature selection.  \\
\citet{Dhillon13} used Randomized Walsh-Hadamard transform to lower the dimension of data matrix and subsequently solve the ridge regression problem in the lower dimensional space. They provided risk-bounds of their algorithm in the fixed design setting. However, this is different from our work, since they use linear combinations of features, while we select actual features from the data.

\section{Our main tools}
\label{sec:main_tool}
\subsection{Single-set Spectral Sparsification}\vskip -0.2cm
We describe the Single-Set Spectral Sparsification algorithm (\textbf{BSS}\footnote{The name BSS comes from the authors Batson, Spielman and Srivastava.} for short) of \citet{BSS09} as Algorithm ~\ref{alg:alg_ssp}. Algorithm ~\ref{alg:alg_ssp} is a greedy technique that selects columns one at a time. Consider the input matrix as a set of $d$ column vectors $\matU^T = \left[ \u_1, \u_2,....,\u_d \right]$, with $\u_i \in \mathbb{R}^\ell \left(i = 1,..,d\right).$ Given $\ell$ and $r>\ell$, we iterate over $\tau = 0,1,2,.. r-1$. Define the parameters $L_{\tau} = \tau - \sqrt{r\ell}, \delta_L = 1, U_{\tau} = \delta_U\left(\tau + \sqrt{\ell r}\right)$ and $\delta_U = \left(1+\sqrt{\ell/r}\right)/\left(1 - \sqrt{\ell/r}\right)$. For $U, L \in \mathbb{R}$ and $\matA \in \mathbb{R}^{\ell\times \ell}$ a symmetric positive definite matrix with eigenvalues $\lambda_1, \lambda_2,...,\lambda_\ell$, define 
$$ \Phi\left(L,\matA\right) = \sum_{i=1}^\ell \frac{1}{\lambda_i-L}; \; \; \hat{\Phi}\left(U,\matA\right) =  \sum_{i=1}^\ell \frac{1}{U-\lambda_i} $$
as the lower and upper potentials respectively. These potential functions measure how far the eigenvalues of $\matA$ are from the upper and lower barriers $U$ and $L$ respectively. We define $\mathcal{L}\left(\u, \delta_L, \matA, L\right)$  and $\mathcal{U}\left(\u, \delta_U, \matA, U\right)$  as follows: 
%
$$ \mathcal{L}\left(\u, \delta_L, \matA, L\right)= \frac{\u^T \left(\matA - \left(L+\delta_L\right)\matI_\ell\right)^{-2}\u}{\Phi\left(L+\delta_L,\matA\right) - \Phi\left(L,\matA\right)} - \u^T\left(\matA - \left(L+\delta_L\right)\matI_\ell\right)^{-1}\u$$ 
%
$$ \mathcal{U}\left(\u, \delta_U, \matA, U\right)= \frac{\u^T \left(\left(U+\delta_U\right)\matI_\ell - \matA\right)^{-2}\u}{ \hat{\Phi}\left(U,\matA\right)-\hat{\Phi}\left(U+\delta_U,\matA\right)} + \u^T\left(\left(U+\delta_U\right)\matI_\ell - \matA\right)^{-1}\u.$$ 
At every iteration, there exists an index $i_{\tau}$ and a weight $t_{\tau}>0$ such that, ${t_{\tau}}^{-1}\leq\mathcal{L}\left(\u_{i_\tau}, \delta_L, \matA, L\right)$ and 
${t_{\tau}}^{-1}\geq \mathcal{U}\left(\u_{i_\tau}, \delta_U, \matA, U\right).$ Thus, there will be at most $r$ columns selected after $\tau$ iterations. The running time of the algorithm is dominated by the search for an index $i_\tau$ satisfying $$ \mathcal{U}\left(\u_{i_\tau},\delta_U,\matA_{\tau},U_{\tau} \right) \leq \mathcal{L}\left(\u_{i_\tau},\delta_L, \matA_{\tau},L_{\tau} \right)$$ and computing the weight $t_{\tau}.$ One needs to compute the upper and lower potentials $\hat{\Phi}\left(U,\matA\right)$ and $\Phi\left(L,\matA\right)$ and hence the eigenvalues of $\matA$. Cost per iteration is $O\left(\ell^3\right)$ and the total cost is $O\left(r\ell^3\right).$ For $i=1,..,d$, we need to compute $\mathcal{L}$ and $\mathcal{U}$ for every $\u_i$ which can be done in $O\left(d\ell^2 \right)$ for every iteration, for a total
of $O\left(rd\ell^2 \right).$ Thus total running time of the algorithm is  $O\left(rd\ell^2 \right).$ 
We present the following lemma for the single-set spectral sparsification algorithm.

\begin{algorithm}[!htb]
\begin{framed}
\textbf{Input:} $\matV^T = [ \v_1, \v_2, ... \v_d ] \in \mathbb{R}^{\ell \times d}$ with $\v_i \in \mathbb{R}^{\ell}$ and $r>\ell$. \\
\textbf{Output:} Matrices $\matS \in \mathbb{R}^{d\times r}, \matD \in \mathbb{R}^{r\times r}$.\\
%
1. Initialize $\matA_0 = \mathbf{0}_{\ell \times \ell}$, $\matS =\mathbf{0}_{d\times r}, \matD =\mathbf{0}_{r\times r}$.\\
2. Set constants $\delta_L = 1$ and $\delta_U = \left(1+\sqrt{\ell/r}\right)/\left(1-\sqrt{\ell/r}\right)$. \\
3. \textbf{for} $\tau = 0$ to $r-1$ \textbf{do}
\begin{itemize}
	\item Let $L_{\tau} = \tau - \sqrt{r\ell} ; U_{\tau} = \delta_U \left(\tau+\sqrt{\ell r}\right)$.
	\item Pick index $i \in \{1,2,..d \}$ and number $t_{\tau}>0$, such that
	$$ \mathcal{U}\left(\v_i,\delta_U,\matA_{\tau},U_{\tau} \right) \leq \mathcal{L}\left(\v_i,\delta_L, \matA_{\tau},L_{\tau} \right). $$
	\item Let $ t_{\tau}^{-1} = \frac{1}{2} \left( \mathcal{U}\left(\v_i,\delta_U,\matA_{\tau},U_{\tau} \right)+ \mathcal{L}\left(\v_i,\delta_L,\matA_{\tau},L_{\tau} \right) \right)$
	\item Update $\matA_{\tau+1} = \matA_{\tau} + t_{\tau}\v_i\v_i^T$ ; set $\matS_{i_\tau,\tau+1}=1$ and $\matD_{\tau+1,\tau+1}=1/\sqrt{t_{\tau}}$.
\end{itemize}
4. \textbf{end for} \\
5. Multiply all the weights in $\matD$ by $\sqrt{r^{-1}\left(1-\sqrt{\left(\ell/r \right)}\right)}.$ \\
6. Return $\matS$ and $\matD.$

\end{framed}
\caption{Single-set Spectral Sparsification}
\label{alg:alg_ssp}
\end{algorithm}

\begin{lemma}\label{lem:bss}
\textbf{BSS} \citep{BSS09}: Given $\matU \in \mathbb{R}^{d \times \ell}$ satisfying $\matU^T\matU = \matI_\ell$ and $r>\ell$, we can deterministically construct sampling and rescaling matrices $\matS\in\mathbb{R}^{r\times d}$ and $\matD\in\mathbb{R}^{r \times r}$ with $\matR=\matD\matS$, such that, for all $\y \in \mathbb{R}^\ell :$
$$ \left(1-\sqrt{\ell/r}\right)^2 \TNormS{\matU\y} \le \TNormS{\matR\matU\y} \le  \left( 1+ \sqrt{\ell/r} \right)^2 \TNormS{\matU\y}.$$
\end{lemma}
We now present a slightly modified version of Lemma~\ref{lem:bss} for our theorems.
\begin{lemma} \label{lem:ssp}
Given $\matU \in \mathbb{R}^{d \times \ell}$ satisfying $\matU^T\matU = \matI_\ell$ and $r>\ell$, we can deterministically construct sampling and rescaling matrices $\matS \in \mathbb{R}^{r\times d}$ and $\matD \in \mathbb{R}^{r \times r}$ such that for $\matR= \matD\matS$, $$\TNorm{ \matU^T\matU - \matU^T\matR^T\matR\matU} \leq 3\sqrt{\ell/r}.$$
\end{lemma}
\begin{proof}
From Lemma~\ref{lem:bss}, it follows,
\begin{equation}
\sigma_\ell\left(\matU^T\matR^T \matR\matU\right)\ge\left(1-\sqrt{\ell/r}\right)^2 \text{ and }
\sigma_1 \left(\matU^T\matR^T\matR\matU\right)\le\left(1+\sqrt{\ell/r}\right)^2. \nonumber
\end{equation}
Thus,
\begin{equation}
\lambda_{max}\left(\matU^T\matU-\matU^T\matR^T \matR\matU\right)\le \left(1-\left(1-\sqrt{\ell/r}\right)^2\right)\le 2\sqrt{\ell/r}.\nonumber
\end{equation}
Similarly,
\begin{equation}
\lambda_{min} \left(\matU^T\matU - \matU^T\matR^T \matR\matU\right) \ge \left( 1 - \left(1 +\sqrt{\ell/r}\right)^2 \right)  \ge 3\sqrt{\ell/r}.\nonumber
\end{equation}
Combining these, we have
$\TNorm{\matU^T\matU - \matU^T\matR^T\matR\matU} \leq 3\sqrt{\ell/r}.$

\noindent Note: Let $\epsilon=3\sqrt{\ell/r}.$ It is possible to set an upper bound on $\epsilon$ by setting the value of $r$. We will assume $\epsilon\in(0,1/2]$.
\end{proof} 

\subsection{Leverage Score Sampling}
Our randomized feature selection method is based on importance sampling or the so-called leverage-score sampling of \citet{Rudelson}. Let $\matU$ be the top-$\rho$ left singular vectors of the training set $\matX$. A carefully chosen probability distribution of the form 
\begin{equation}
p_i = \frac{\TNormS{\matU_{i}}}{n}, \text{ for } i=1,2,...,d, 
\label{eqn:eqnlvg}
\end{equation}
i.e. proportional to the squared Euclidean norms of the rows of the left-singular vectors and select $r$ rows of $\matU$ in i.i.d trials and re-scale the rows with $1/\sqrt{p_i}$. The time complexity is dominated by the time to compute the SVD of $\matX$.

\begin{lemma} \label{lem:lvgscr}
\citep{Rudelson} Let $\epsilon \in(0,1/2]$ be an accuracy parameter and $\delta \in(0,1)$ be the failure probability. Given $\matU \in \mathbb{R}^{d \times \ell}$ satisfying $\matU^T\matU = \matI_\ell.$ Let $\tilde{p} = min\{1, rp_i\}$, let $p_i$ be as Eqn.~\ref{eqn:eqnlvg} and let $r = O\left(\frac{n}{\epsilon^2} \log\left(\frac{n}{\epsilon^2 \sqrt{\delta}}\right) \right)$. Construct the sampling and rescaling matrix $\matR$.  Then with probability at least $(1-\delta)$, 
$ \TNorm{\matU^T\matU - \matU^T\matR^T\matR\matU} \leq \epsilon.$
\end{lemma}

\section{Theory}
In this section we describe the theoretical guarantees of RLSC using BSS and also the risk bounds of ridge regression using BSS and Leverage-score sampling. Before we begin, we state the following lemmas from numerical linear algebra which will be required for our proofs. 
\begin{lemma}\citep{stewart}
For any matrix $\matE$, such that $\matI+\matE$ is invertible,
$\left(\matI+\matE\right)^{-1}=\matI+\sum\limits_{i=1}^\infty(-\matE)^i.$
\label{lem:leminv}
\end{lemma}
\begin{lemma}\citep{stewart}
Let $\matA$ and $\tilde{\matA}= \matA+\matE$ be invertible matrices. Then $$\tilde{\matA}^{-1} - {\matA}^{-1} = -{\matA}^{-1}\matE\tilde{\matA}^{-1}.$$
\label{lem:leminv2}
\end{lemma}
\begin{lemma}\citep{demmel}
Let $\matD$ and $\matX$ be matrices such that the product $\matD\matX\matD$ is a symmetric positive definite matrix with matrix $\matX_{ii}=1$. Let the product $\matD\matE\matD$ be a perturbation such that, $\TNorm{E} = \eta < \lambda_{min}(\matX).$ Here  $\lambda_{min}$ corresponds to the smallest eigenvalue of $\matX$. Let $\lambda_i$ be the i-th eigenvalue of $\matD\matX\matD$ and let $\tilde{\lambda}_i$ be the i-th eigenvalue of $\matD\left(\matX+\matE\right)\matD.$ Then, 
$\abs{\frac{\lambda_i - \tilde{\lambda}_i}{\lambda_i}} \leq \frac{\eta}{\lambda_{min}\left(\matX\right)}.$
\label{lem:lem5}
\end{lemma}
\begin{lemma}
Let $\epsilon\in(0,1/2]$. Then $\TNorm{\q^T\matU^{\perp}\matU^{\perp T}\matR^T \matR\matU} \leq \epsilon\TNorm{\matU^{\perp}\matU^{\perp T}\q}.$
\label{lem:per3}
\end{lemma}
The proof of this lemma is similar to Lemma 4.3 of \cite{Drineas06}. 
\subsection{Our Main Theroems on RLSC}
The following theorem shows the additive error guarantees of the generalization bounds of the approximate classifer with that of the classifier with no feature selection. The classification error bound of BSS on RLSC depends on the condition number of the training set and on how much of the test-set lies in the subspace of the training set.
\begin{theorem} 
\label{thm:Thm1}
Let $\epsilon\in(0,1/2]$ be an accuracy parameter, $r=O\left(n/\epsilon^2\right)$ be the number of features selected by BSS.
Let \math{\matR\in\mathbb{R}^{r\times d}} be the matrix, as defined in Lemma~\ref{lem:ssp}. Let $\matX \in \mathbb{R}^{d\times n}$ with $d>> n$, be the training set, $\tilde{\matX} =\matR\matX$ is the reduced dimensional matrix and $\q \in \mathbb{R}^d$ be the test point of the form $\q =  \matX\bm{\alpha} + \matU^{\perp}\bm{\beta}$. Then, the following hold:
\begin{itemize}
\item If $\lambda=0$, then $\abs{\tilde{\q}^T\tilde{\matX}\tilde{x}_{opt} - \q^T\matX\x_{opt}} \leq \frac{\epsilon \kappa_{\matX}}{\sigma_{max}} \TNorm{\bm\beta}\TNorm{\y}.$
\item If $\lambda>0$, then $\abs{\tilde{\q}^T\tilde{\matX}\tilde{x}_{opt} - \q^T\matX\x_{opt}} \leq 2\epsilon\kappa_{\matX}\TNorm{\bm{\alpha}}\TNorm{\y} +  \frac{2\epsilon \kappa_{\matX}}{\sigma_{max}} \TNorm{\bm\beta}\TNorm{\y}.$
\end{itemize}
\end{theorem}
\begin{proof}
We assume that $\matX$ is a full-rank matrix.
Let $\matE = \matU^T\matU - \matU^T \matR^T\matR\matU$ and $\TNorm{\matE}=\TNorm{\matI - \matU^T \matR^T\matR\matU} = \epsilon\leq 1/2$.
Using the SVD of $\matX$, we define
\begin{equation}
\matDelta = \matSig\matU^T\matR^T\matR\matU\matSig = \matSig\left(\matI+\matE\right)\matSig.
\label{eqn:delta}
\end{equation}
The optimal solution in the sampled space is given by,
\begin{equation}
\tilde{\x}_{opt} = \matV\left(\matDelta+\lambda\matI\right)^{-1}\matV^T\y. 
\label{eqn:xopt3}
\end{equation}
\noindent
It can be proven easily that $\matDelta$ and $\matDelta+\lambda\matI$ are invertible matrices. 
We focus on the term $\q^T\matX\x_{opt}.$ Using the SVD of $\matX$, we get
\begin{eqnarray}
 \q^T\matX\x_{opt} &=& \bm\alpha^T\matX^T\matX\x_{opt} + \bm\beta\matU^{\perp T}\left(\matU\matSig\matV^T\right)\x_{opt} \nonumber \\
&=&  \bm\alpha^T\matV \matSig^2\left(\matSig^2+\lambda\matI\right)^{-1} \matV^T\y \label{eqn:eqn8} \\
&=&  \bm\alpha^T\matV \left(\matI + \lambda\matSig^{-2}\right)^{-1}\matV^T\y. \label{eqn:eqn9}
\end{eqnarray}
Eqn(\ref{eqn:eqn8}) follows because of the fact $\matU^{\perp T}\matU= \bm0$ and by substituting  $\x_{opt}$ from Eqn.(\ref{eqn:xopt}). Eqn.(\ref{eqn:eqn9}) follows from the fact that the matrices $\matSig^2$ and $\matSig^2+\lambda\matI$ are invertible. Now,
\begin{eqnarray}
\abs{\q^T\matX\x_{opt} - \tilde{\q}^T\tilde{\matX}\tilde{\x}_{opt}} &=& \abs{ \q^T\matX\x_{opt} - \q^T\matR^T\matR\matX\tilde{\x}_{opt}} \nonumber \\
&\leq&  \abs{ \q^T\matX\x_{opt} - \bm\alpha^T \matX^T\matR^T \matR\matX\tilde{\x}_{opt}} \label{eqn:term1} \\
&& + \abs{\bm\beta^T \matU^{{\perp}T}\matR^T\matR\matX\tilde{\x}_{opt} }. \label{eqn:term2}
\end{eqnarray}
\noindent
We bound (\ref{eqn:term1}) and (\ref{eqn:term2}) separately. Substituting the values of $\tilde{\x}_{opt}$ and $\matDelta$, 
\begin{eqnarray}
\bm\alpha^T \matX^T\matR^T \matR\matX\tilde{\x}_{opt} &=& \bm\alpha^T\matV\matDelta\matV^T\tilde{\x}_{opt} \nonumber \\
&=&   \bm\alpha^T \matV\matDelta\left(\matDelta + \lambda\matI \right)^{-1}\matV^T\y \nonumber \\
&=&   \bm\alpha^T \matV\left(\matI+\lambda\matDelta^{-1}\right)^{-1} \matV^T\y \nonumber \\
&=&   \bm\alpha^T \matV \left(\matI+ \lambda\matSig^{-1} \left(\matI+\matE\right)^{-1} \matSig^{-1} \right)^{-1} \matV^T \y \nonumber \\
&=&   \bm\alpha^T \matV \left(\matI+ \lambda\matSig^{-2} + \lambda\matSig^{-1}\matPhi\matSig^{-1} \right)^{-1} \matV^T \y. \label{eqn:eqnA}
\end{eqnarray}
\noindent
The last line follows from Lemma \ref{lem:leminv} in Appendix, which states that $\left(\matI+\matE\right)^{-1} = \matI+\matPhi$, where $\matPhi = \sum\limits_{i=1}^\infty (-\matE)^i$. The spectral norm of $\matPhi$ is bounded by,
\begin{equation}
\TNorm{\matPhi} = \TNorm{\sum\limits_{i=1}^\infty (-\matE)^i} \leq \sum\limits_{i=1}^\infty \TNorm{\matE}^i \leq \sum\limits_{i=1}^\infty \epsilon^i = \epsilon/(1-\epsilon).
\label{eqn:phi} 
\end{equation}
We now bound (\ref{eqn:term1}). Substituting (\ref{eqn:eqn9}) and (\ref{eqn:eqnA}) in  (\ref{eqn:term1}), 
\begin{eqnarray}
&& \abs{\q^T\matX\x_{opt} - \bm\alpha^T \matX^T\matR^T \matR\matX\tilde{\x}_{opt}} \nonumber \\
&=& \abs{\bm\alpha^T\matV\{\left(\matI+\lambda\matSig^{-2}+\lambda\matSig^{-1}\matPhi\matSig^{-1} \right)^{-1}-\left(\matI+\lambda\matSig^{-2}\right)^{-1}\}\matV^T\y} \nonumber \\
&\leq& \TNorm{\bm\alpha^T\matV\left(\matI+\lambda\matSig^{-2}\right)} \TNorm{\matV^T\y}\TNorm{\matPsi}. \nonumber 
\end{eqnarray} 
\noindent
The last line follows because of Lemma \ref{lem:leminv2} and the fact that all matrices involved are invertible. Here,
\begin{eqnarray}
\matPsi &=& \lambda\matSig^{-1}\matPhi\matSig^{-1}\left(\matI+ \lambda\matSig^{-2}+\lambda\matSig^{-1}\matPhi\matSig^{-1}\right)^{-1} \nonumber \\
&=&  \lambda\matSig^{-1}\matPhi \matSig^{-1}\left(\matSig^{-1}\left(\matSig^2+\lambda\matI+\lambda\matPhi\right)\matSig^{-1}\right)^{-1} \nonumber \\
&=&  \lambda\matSig^{-1}\matPhi\left(\matSig^2+\lambda\matI+\lambda\matPhi\right)^{-1}\matSig. \nonumber
\end{eqnarray}
Since the spectral norms of $\matSig, \matSig^{-1}$ and $\matPhi$ are bounded, we only need to bound the spectral norm of $\left(\matSig^2+\lambda\matI+\lambda\matPhi\right)^{-1}$ to bound the spectral norm of $\matPsi$. The spectral norm of the matrix $\left(\matSig^2+\lambda\matI+\lambda\matPhi\right)^{-1}$ is the inverse of the smallest singular value of $\left(\matSig^2+\lambda\matI+\lambda\matPhi\right).$ From perturbation theory of matrices \cite{stewart} and (\ref{eqn:phi}), we get
$$\abs{\sigma_i \left(\matSig^2 +\lambda\matI+\lambda\matPhi\right) - \sigma_i \left(\matSig^2+\lambda\matI\right)}\leq \TNorm{\lambda\matPhi}\leq \epsilon\lambda.$$ Here, $\sigma_i(\matQ)$ represents the $i^{th}$ singular value of the matrix $\matQ$. \\
Also, ${\sigma_i}^2 \left(\matSig^2+\lambda\matI\right) = {\sigma_i}^2+\lambda,$ where $\sigma_i$ are the singular values of $\matX$.
$${\sigma_i}^2+(1-\epsilon)\lambda \leq \sigma_i \left(\matSig^2+\lambda\matI+\lambda\matPhi\right) \leq {\sigma_i}^2+(1+\epsilon)\lambda.$$
Thus, $$\TNorm{\left(\matSig^2+\lambda\matI+\lambda\matPhi\right)^{-1}} = 1/\sigma_{min}\left(\matSig^2+\lambda\matI+\lambda\matPhi\right) \leq 1/\left( {\sigma^2}_{min} + (1-\epsilon)\lambda)\right).$$
Here, $\sigma_{max}$ and $\sigma_{min}$ denote the largest and smallest singular value of $\matX$. Since $\TNorm{\matSig}\TNorm{\matSig^{-1}} = \sigma_{max}/\sigma_{min} \leq \kappa_{\matX}$, (condition number of $\matX$) we bound (\ref{eqn:term1}):\\
\begin{equation}
\abs{\q^T\matX \x_{opt}-\bm\alpha^T \matX^T \matR^T \matR\matX \tilde{\x}_{opt}} \leq \frac{\epsilon\lambda\kappa_{\matX}}{{\sigma^2}_{min}+(1-\epsilon)\lambda} \TNorm{\bm\alpha^T \matV\left(\matI+\lambda\matSig^{-2}\right)^{-1}} \TNorm{\matV^T \y}.
\label{eqn:term1_bound}
\end{equation}
For $\lambda>0$, the term ${\sigma^2}_{min}+(1-\epsilon)\lambda$ in Eqn.(\ref{eqn:term1_bound}) is always larger than $\left(1-\epsilon\right)\lambda$, so it can be upper bounded by $2\epsilon\kappa_{\matX}$ (assuming $\epsilon\leq 1/2$). Also,
%
$$\TNorm{\bm\alpha^T \matV\left(\matI+\lambda\matSig^{-2}\right)^{-1}} \leq \TNorm{\bm\alpha^T \matV}\TNorm{\left(\matI+\lambda\matSig^{-2}\right)^{-1}} \leq \TNorm{\bm\alpha}.$$
This follows from the fact, that
$\TNorm{\bm\alpha^T\matV}=\TNorm{\bm\alpha}$ and $\TNorm{\matV\y}=\TNorm{\y}$ as $\matV$ is a full-rank orthonormal matrix and the singular values of 
$\matI+\lambda\matSig^{-2}$ are equal to $1+\lambda/{\sigma_i}^2$; making the spectral norm of its inverse at most one. Thus we get,
\begin{equation}
\abs{\q^T\matX \x_{opt}-\bm\alpha^T \matX^T \matR^T \matR\matX \tilde{\x}_{opt}} \leq 2\epsilon\kappa_{\matX}\TNorm{\bm\alpha}\TNorm{\y}.
\label{eqn:bound1}
\end{equation}
We now bound (\ref{eqn:term2}). Expanding (\ref{eqn:term2}) using SVD and $\tilde{\x}_{opt}$,
\begin{eqnarray}
\abs{\bm\beta^T \matU^{{\perp}T}\matR^T\matR\matX \tilde{\x}_{opt}} &=&
\abs{\bm\beta^T \matU^{{\perp}T}\matR^T\matR\matU\matSig\left(\matDelta+\lambda\matI\right)\matV^T \y} \nonumber \\ 
&\leq& \TNorm{\q^T\matU^{\perp}\matU^{\perp T}\matR^T\matR\matU}\TNorm{\matSig\left(\matDelta+\lambda\matI\right)^{-1}} \TNorm{\matV^T \y} \nonumber \\
&\leq& \epsilon \TNorm{\matU^{\perp}\matU^{\perp T}\q} \TNorm{\matV^T \y} \TNorm{\matSig\left(\matDelta+\lambda\matI\right)^{-1}} \nonumber \\
&\leq& \epsilon \TNorm{\bm\beta}\TNorm{\y} \TNorm{\matSig\left(\matDelta+\lambda\matI\right)^{-1}}. \nonumber
\end{eqnarray}
The first inequality follows from $\bm\beta = \matU^{\perp T}\q$; and the second inequality follows from Lemma \ref{lem:per3}. 
To conclude the proof, we bound the spectral norm of $\matSig\left(\matDelta+\lambda\matI\right)^{-1}$. Note that from Eqn.(\ref{eqn:delta}),  $\matSig^{-1}\matDelta\matSig^{-1} = \matI +\matE$ and $\matSig\matSig^{-1}=\matI$,
$$\matSig\left(\matDelta+\lambda\matI\right)^{-1}=\left(\matSig^{-1}\matDelta\matSig^{-1}+\lambda\matSig^{-2}\right)^{-1} \matSig^{-1}= \left(\matI+\lambda\matSig^{-2}+\matE\right)^{-1} \matSig^{-1}.$$
\noindent
One can get a lower bound for the smallest singular value of $\left(\matI+\lambda\matSig^{-2}+\matE\right)^{-1}$ using matrix perturbation theory and by comparing the singular values of this matrix to the singular values of $\matI +\lambda\matSig^{-2}.$ 
We get,
$$\left(1-\epsilon\right)+\frac{\lambda}{{\sigma_i}^2} \leq {\sigma_i}\left(\matI+\matE+\lambda\matSig^{-2}\right) \leq \left(1+\epsilon\right)+\frac{\lambda}{{\sigma_i}^2}.$$ 
\begin{eqnarray}
\TNorm{\left(\matI+\lambda\matSig^{-2}+\matE\right)^{-1} \matSig^{-1}} &\leq& \frac{{\sigma^2}_{max}}{\left(\left(1-\epsilon\right){\sigma^2}_{max}+\lambda   \right){\sigma_{min}}} \nonumber \\
&=& \frac{\kappa_{\matX}\sigma_{max}}{\left(1-\epsilon\right){\sigma^2}_{max}+\lambda} \nonumber \\
&\leq&  \frac{2\kappa_{\matX}}{\sigma_{max}}. 
\end{eqnarray}
We assumed that $\epsilon\leq 1/2$, which implies $\left(1-\epsilon\right)+\lambda/{\sigma^2}_{max}\geq 1/2.$ Combining these, we get, 
\begin{equation}
\abs{\bm\beta^T \matU^{{\perp}T}\matR^T\matR\matX\tilde{\x}_{opt}} \leq \frac{2\epsilon\kappa_{\matX}}{\sigma_{max}}\TNorm{\bm{\beta}}\TNorm{\y}.
\label{eqn:bound2}
\end{equation}
Combining Eqns (\ref{eqn:bound1}) and (\ref{eqn:bound2}) we complete the proof for the case $\lambda>0$. For $\lambda=0$, Eqn.(\ref{eqn:term1_bound}) becomes zero and the result follows. 
\end{proof}
\noindent
Our next theorem provides relative-error guarantees to the bound on the classification error when the test-point has no-new components, i.e. $\bm\beta=\bm0.$
\begin{theorem}
Let $\epsilon\in(0,1/2]$ be an accuracy parameter, $r=O\left(n/\epsilon^2\right)$ be the number of features selected by BSS and $\lambda>0$. Let $\q \in \mathbb{R}^d$ be the test point of the form $\q =  \matX\bm{\alpha}$, i.e. it lies entirely in the subspace spanned by the training set, and the two vectors $\matV^T\y$ and $\left(\matI+\lambda\matSig^{-2}\right)^{-1}\matV^T\bm\alpha$ satisfy the property,
\begin{eqnarray}
\TNorm{\left(\matI+\lambda\matSig^{-2}\right)^{-1}\matV^T\bm\alpha}\TNorm{\matV^T\y} &\leq& \omega \TNorm{\left( \left(\matI+\lambda\matSig^{-2}\right)^{-1}\matV^T\bm\alpha\right)^T \matV^T\y} \nonumber \\
&=& \omega\abs{\q^T\matX \x_{opt}} \nonumber
\end{eqnarray}
for some constant $\omega$. If we run RLSC after BSS, then $$\abs{\tilde{\q}^T\tilde{\matX}\tilde{x}_{opt}-\q^T\matX\x_{opt}} \leq 2\epsilon\omega\kappa_{\matX}  \abs{\q^T\matX\x_{opt}}.$$
\label{thm:thm2}
\end{theorem}
\noindent
The proof follows directly from the proof of Theorem 1 if we consider $\bm\beta=\bm0$.
\subsection{Our Main Theorems on Ridge Regression}
We compare the risk of subsampled ridge regression with the risk of true dual ridge regreesion in the fixed design setting.
Recall that the response vector $\y =\matX^T \betavec+\omegavec$ where $\omegavec \in \mathbb{R}^n$ is the homoskedastic noise vector with mean 0 and variance $\sigma^2$.
Also, we assume that the data matrix is of full rank.
\begin{lemma}
\label{lem:lembss}
Let $\rho$ be the rank of $\matX$. Form $\tilde{\matK}$ using BSS. Then,
\begin{equation}
\left(1-\Delta \right)\matK \preceq \tilde{\matK} \preceq \left(1+\Delta \right) \matK,\nonumber
\end{equation}
where $\Delta = C\sqrt{\rho/r}.$ For p.s.d matrices $\matA \succeq \matB$ means $\matB -\matA$ is p.s.d. 
\end{lemma}
\begin{proof}
Using the SVD of $\matX$, $\tilde{\matK} = \matV\matSig\left(\matU^T\matR^T\matR\matU\right)\matSig\matV^T$. Lemma~\ref{lem:ssp} implies $$\matI_\rho \left(1-\Delta\right) \preceq \left(\matU^T \matR^T \matR\matU \right) \preceq \matI_\rho \left(1+\Delta \right).$$
Multiplying left and right hand side of the inequality by $\matV\matSig$ and $\matSig \matV^T$ respectively, to the above inequality completes the proof.
\end{proof}
\begin{lemma}
\label{lem:lemlvg}
Let $\rho$ be the rank of $\matX$. Form $\tilde{\matK}$ using leverage-score sampling. Then, with probability at least $(1-\delta)$, where $\delta \in (0,1)$,
\begin{equation}
\left(1-\Delta \right)\matK \preceq \tilde{\matK} \preceq \left(1+\Delta \right) \matK,\nonumber
\end{equation}
where $\Delta = C\frac{\rho}{\epsilon^2} \log\left(\frac{\rho}{\epsilon^2 \sqrt{\delta}}\right).$ 
\end{lemma}

\subsubsection{Risk Function for Ridge Regression}
Let $\z = \bf{E}_\omega [\y] = \matX^T \betavec$. The risk for a prediction function $\hat\y \in \mathbb{R}^n$ is $\frac{1}{n}\bf{E}_\omega {\TNormS{\hat\y -\z}}.$ For any $n\times n$ positive symmetric matrix $\matK$, we define the following risk function:
\begin{equation}
R\left(\matK\right) = \frac{\sigma^2}{n} \Trace{\matK^2 \left(\matK+ n\lambda\matI_n \right)^{-2}}
+ n\lambda^2 \z^T\left(\matK+n\lambda\matI_n \right)^{-2}\z.\nonumber 
\end{equation}
%
\begin{theorem}
\label{thm:thm1}
Under the fixed design setting, the risk for the ridge regression solution in the full-feature space is $R(\matK)$ and the risk for the ridge regression in the reduced dimensional space is $R(\tilde\matK).$
\end{theorem}
\begin{proof}
The risk of the ridge regression estimator in the reduced dimensional space is
\begin{equation}
\frac{1}{n}\textbf{E}_\omega\TNormS{\matKt\tilde\a_\lambda-\z} = \frac{1}{n}\textbf{E}_\omega\TNormS{\matKt\left(\matKt+n\lambda\matI_n\right)^{-1}\y -\z }. 
\label{eqn:riskrr}
\end{equation}
Taking $\matKt\left(\matKt+n\lambda\matI_n\right)^{-1}$ as $\matQ$ we can write Eqn.(\ref{eqn:riskrr}) as,
\begin{eqnarray*}
&&\frac{1}{n} \textbf{E}_\omega \TNormS{\matQ\y - \textbf{E}_\omega\left[ \matQ \y \right]} + \frac{1}{n} \TNormS{\textbf{E}_\omega\left[\matQ\y \right]-\z}  \nonumber \\
&=&\frac{1}{n}\textbf{E}_\omega \left[\TNormS{\matKt\left(\matKt+n\lambda\matI_n\right)^{-1} \omegavec}\right] 
+\frac{1}{n} \TNormS{\matKt\left(\matKt+ n \lambda\matI_n\right)^{-1}\z-\z} \nonumber \\ 
&=& \frac{1}{n}\Trace{\matKt^2 \left(\matKt+n \lambda\matI_n\right)^{-2}\omegavec\omegavec^T}+ \frac{1}{n}\z^T\left(\matI_n-\matKt\left(\matKt+n\lambda\matI_n\right)^{-1}\right)^2 \z \nonumber \\
&=& \frac{\sigma^2}{n}\Trace{\matKt^2 \left(\matKt+n\lambda\matI_n\right)^{-2}}+ n\lambda^2 \z^T \left(\matKt+n\lambda\matI_n\right)^{-2}\z. \nonumber
\end{eqnarray*}
The expectation is only over the random noise $\omegavec$ and is conditional on the feature selection method used.
\end{proof}
Our next theorem bounds the risk inflation of ridge regression in the reduced dimensional space compared with the ridge regression solution in the full-feature space.
\begin{theorem}
\label{thm:thm2}
Let $\rho$ be the rank of the matrix $\matX$. When using leverage-score sampling as a feature selection technique, with probability 
at least $1-\delta$, where $\delta\in(0,1)$,
$$R(\matKt) \leq (1-\Delta)^{-2}R(\matK),$$ where $\Delta=C\frac{\rho}{\epsilon^2} \log\left(\frac{\rho}{\epsilon^2 \sqrt{\delta}}\right).$
\end{theorem}
\begin{proof}
For any positive semi-definite matrix, $\matK \in \mathbb{R}^{n\times n}$, we define the bias $B(\matK)$ and variance $V(\matK)$ of the risk function as follows:
$$ B(\matK) = n\lambda^2 \z^T\left(\matK+ n\lambda\matI_n \right)^{-2}\z,$$
$$ V(\matK) = \frac{\sigma^2}{n}\Trace{\matKt^2 \left(\matKt+n\lambda\matI_n\right)^{-2}}.$$
Therefore, $R(\matK) = B(\matK) + V(\matK).$ Now due to \cite{Bach13} we know $B(\matK)$ is non-increasing in $\matK$ and $V(\matK)$ is non-decreasing in $\matK$. When Lemma ~\ref{lem:lemlvg} holds,
\begin{eqnarray}
R(\matKt) &=& V(\matKt) + B(\matKt) \nonumber \\
&\leq& V\left(\left(1+\Delta\right)\matK\right) + B\left(\left(1-\Delta\right)\matK\right) \nonumber \\
&\leq& \left(1+\Delta\right)^2 V(\matK) + \left(1-\Delta\right)^{-2} B(\matK) \nonumber \\
&\leq&  \left(1-\Delta\right)^{-2}\left(V(\matK)+B(\matK)\right) \nonumber \\
&=&  \left(1-\Delta\right)^{-2} R(\matK). \nonumber
\end{eqnarray}
\end{proof}
We can prove a similar theorem for BSS.
\begin{theorem}
\label{thm:thm2}
Let $\rho$ be the rank of the matrix $\matX$. When using BSS as a feature selection technique, with $\Delta=C\rho/\epsilon^2,$
$$R(\matKt) \leq (1-\Delta)^{-2}R(\matK).$$ 
\end{theorem}
\section{Experiments}
All experiments were performed in MATLAB R2013b on an Intel i-7 processor with 16GB RAM.
\subsection{BSS Implementation Issues}
The authors of ~\cite{BSS09} do not provide any implementation details of the \textbf{BSS} algorithm. Here we discuss several issues arising during the implementation. \\
\noindent\textbf{Choice of column selection:} 
At every iteration, there are multiple columns which satisfy the condition $\mathcal{U}\left(\u_i,\delta_U,\matA_{\tau},U_{\tau} \right) \leq \mathcal{L}\left(\u_i,\delta_L, \matA_{\tau},L_{\tau} \right).$  The authors of ~\cite{BSS09} suggest picking any column which satisfies this constraint. Instead of breaking ties arbitrarily, we choose the column $\u_i$ which has not been selected in previous iterations and whose Euclidean-norm is highest among the candidate set. Columns with zero Euclidean norm never get selected by the algorithm. 
In the inner loop of Algorithm~\ref{alg:alg_ssp}, $\mathcal{U}$ and $\mathcal{L}$ has to be computed for all the $d$ columns in order to pick a good column. This step can be done efficiently using a single line of Matlab code, by making use of matrix and vector operations. 
\subsection{Other Feature Selection Methods}
In this section, we describe other feature-selection methods with which we compare BSS.

\subsubsection{Rank-Revealing QR Factorization (RRQR)}
Within the numerical linear algebra community, subset selection algorithms use the so-called Rank Revealing QR (RRQR) factorization. Here we slightly abuse notation and state $\matA$ as a short and fat matrix as opposed to the tall and thin matrix.
Let $\matA$ be a $n\times d$ matrix with $\left(n<d\right)$ and an integer $k \left(k<d\right)$ and assume partial QR factorizations of the form $$\matA\matP = \matQ \begin{pmatrix} \matR_{11} & \matR_{12} \\ \mathbf{0} & \matR_{22} \end{pmatrix},$$ 
where $\matQ \in \mathbb{R}^{n\times n}$ is an orthogonal matrix, $\matP  \in \mathbb{R}^{d\times d}$ is a permutation matrix, $\matR_{11} \in \mathbb{R}^{k\times k}, \matR_{12} \in \mathbb{R}^{k\times (d-k)},  \matR_{22} \in \mathbb{R}^{(d-k)\times (d-k)}$ The above factorization is called a RRQR factorization if $\sigma_{min}\left( \matR_{11}\right) \geq \sigma_k\left(\matA\right)/p(k,d)$, $\sigma_{max}\left( \matR_{22}\right) \leq \sigma_{min}(\matA) p(k,d),$ where $p(k,d)$ is a function bounded by a low-degree polynomial in $k$ and $d$. The important columns are given by $\matA_1= \matQ \begin{pmatrix} \matR_{11} \\ \mathbf{0} \end{pmatrix}$ and $\sigma_i \left(\matA_1\right) = \sigma_i\left(\matR_{11}\right)$ with $1\leq i \leq k.$ We perform feature selection using RRQR by picking the important columns which preserve the rank of the matrix.

\subsubsection{Random Feature Selection}
We select features uniformly at random without replacement which serves as a baseline method. To get around the randomness, we repeat the sampling process five times. 

\subsubsection{Leverage-Score Sampling}
 For leverage-score sampling, we repeat the experiments five times to get around the randomness. We pick the top-$\rho$ left singular vectors of $\matX,$ where $\rho$ is the rank of the matrix $\matX.$

\subsubsection{Information Gain (IG)} The Information Gain feature selection method \citep{yang97} measures the amount of information obtained for binary class prediction by knowing the presence or absence of a feature in a dataset. The method is a supervised strategy, whereas the other methods used here are unsupervised.

\begin{table}[!htbp]
\caption{\small Most frequently selected features using the synthetic dataset.}
\label{tab:synth}
\begin{center}
\begin{small}
\begin{tabular}{|c||c|c|}
\hline
$r=80$   & $k=90$ & $k=100$ \\
\hline
BSS	&89, 88, 87, 86, 85 &100, 99, 98, 97, 95 \\
\hline
RRQR   &90, 80, 79, 78, 77  & 100, 80, 79, 78, 77  \\
\hline
Lvg-Score &73, 85, 84, 81, 87   &93, 87, 95, 97, 96   \\
\hline
IG &80, 79, 78, 77, 76 &80, 79, 78, 77, 76 \\
\hline\hline
$r=90$   & $k=90$ & $k=100$ \\
\hline
BSS	& 90, 88, 87, 86, 85  &100, 99, 98, 97, 96   \\
\hline
RRQR   & 90, 89, 88, 87, 86  &100, 90, 89, 88, 87  \\
\hline
Lvg-Score &67, 88, 83, 87, 85 &100, 97, 92, 48, 58 \\
\hline
IG &90, 89, 88, 87, 86 &90, 89, 88, 87, 86 \\
\hline
\end{tabular}
\end{small}
\end{center}
\end{table}

\begin{table}[!htb]
\caption{\small Running time of various feature selection methods in seconds. For synthetic data, the running time corresponds to the experiment when $r=80$ and $k=90$ and is averaged over ten ten-fold cross-validation experiments. For TechTC-300, the running time corresponds to the experiment when $r=400$ and is averaged over ten ten-fold cross-validation experiments and over 48 TehTC-300 datasets.}
\label{tab:rlsc_trun}
\begin{center}
\begin{small}
\begin{tabular}{|c||c|c|c|c|}
\hline
 &BSS &IG &LVG &RRQR \\ \hline
Synthetic Data &0.1025  &0.0003    &0.0031    &0.0016 \\ \hline
TechTC-300  &75.7624      &0.0242    &0.4054   &0.2631 \\ \hline
\end{tabular}
\end{small}
\end{center}
\end{table}

\subsection{Experiments on RLSC}
The goal of this section is to compare BSS with existing feature selection methods for RLSC and show that BSS is better than the other methods.
\subsubsection{Synthetic Data}\label{subsubsec:synth}
We run our experiments on synthetic data where we control the number of relevant features in the dataset and demonstrate the working of Algorithm~\ref{alg:alg_ssp} on RLSC.
We generate synthetic data in the same manner as given in \cite{Bhat04}. The dataset has $n$ data-points and $d$ features. The class label $y_i$ of each data-point was randomly chosen to be 1 or -1 with equal probability. The first $k$ features of each data-point $\x_i$ are drawn from $y_i \mathcal{N}\left(-j,1\right)$ distribution, where $\mathcal{N}\left(\mu,\sigma^2\right)$ is a random normal distribution with mean $\mu$ and variance $\sigma^2$ and $j$ varies from 1 to k. The remaining $d-k$ features are chosen from a $\mathcal{N}(0,1)$ distribution. Thus the dataset has $k$ relevant features and $(d-k)$ noisy features. By construction, among the first $k$ features, the $kth$ feature has the most discriminatory power, followed by $(k-1)th$ feature and so on. We set $n$ to 30 and $d$ to 1000. We set $k$ to 90 and 100 and ran two sets of experiments.\\
\begin{table}[!htbp]
\caption{\small Out-of-sample error of TechTC-300 datasets averaged over ten ten-fold cross-validation and over 48 datasets for three values of $r$. The first and second entry of each cell represents the mean and standard deviation. Items in bold indicate the best results.}
\label{tab:techtc_alleout}
\begin{center}
\begin{small}
\begin{tabular}{|c|c|c|c|c|}
\hline
$r=300$ & $\lambda=0.1$ & $\lambda=0.3$ & $\lambda=0.5$ & $\lambda=0.7$ \\
\hline
\textbf{BSS}  &\textbf{31.76} $\pm$ \textbf{0.68}     & \textbf{31.46} $\pm$ \textbf{0.67}    & \textbf{31.24} $\pm$ \textbf{0.65}    &\textbf{31.03} $\pm$ \textbf{0.66}  \\
\hline
 \textbf{Lvg-Score}   &38.22 $\pm$ 1.26   &37.63 $\pm$ 1.25   &37.23 $\pm$ 1.24  &36.94 $\pm$ 1.24 \\
\hline
\textbf{RRQR} &37.84 $\pm$ 1.20    &37.07 $\pm$ 1.19  &36.57 $\pm$ 1.18  &36.10 $\pm$ 1.18 \\
\hline
\textbf{Randomfs}  &50.01 $\pm$ 1.2  &49.43 $\pm$ 1.2  &49.18 $\pm$ 1.19   &49.04 $\pm$ 1.19 \\
\hline 
\textbf{IG}  &38.35  $\pm$ 1.21  &36.64  $\pm$ 1.18    &35.81  $\pm$ 1.18   &35.15  $\pm$ 1.17 \\
\hline \hline
$r=400$ & $\lambda=0.1$ & $\lambda=0.3$ & $\lambda=0.5$ & $\lambda=0.7$ \\
\hline
\textbf{BSS}  &\textbf{30.59}  $\pm$ \textbf{0.66}  &\textbf{30.33}  $\pm$ \textbf{0.65}  &\textbf{30.11}  $\pm$ \textbf{0.65}  &\textbf{29.96}  $\pm$ \textbf{0.65}  \\ 
\hline
\textbf{Lvg-Score}    &35.06 $\pm$ 1.21   &34.63 $\pm$ 1.20  &34.32 $\pm$ 1.2  &34.11 $\pm$ 1.19 \\
\hline
\textbf{RRQR}    &36.61 $\pm$ 1.19   &36.04  $\pm$ 1.19  &35.46  $\pm$ 1.18  &35.05  $\pm$ 1.17    \\
\hline
\textbf{Randomfs}  &47.82 $\pm$ 1.2   &47.02 $\pm$ 1.21  &46.59 $\pm$ 1.21   &46.27 $\pm$ 1.2 \\
\hline 
\textbf{IG}  &37.37  $\pm$ 1.21  &35.73  $\pm$ 1.19    &34.88  $\pm$ 1.18   &34.19  $\pm$ 1.18 \\
\hline\hline
$r=500$ & $\lambda=0.1$ & $\lambda=0.3$ & $\lambda=0.5$ & $\lambda=0.7$ \\
\hline
\textbf{BSS}    &\textbf{29.80} $\pm$ \textbf{0.77}  &\textbf{29.53} $\pm$ \textbf{0.77}   &\textbf{29.34} $\pm$ \textbf{0.76}  &\textbf{29.18} $\pm$ \textbf{0.75} \\ 
\hline
\textbf{Lvg-Score}    &33.33 $\pm$ 1.19  &32.98 $\pm$ 1.18  &32.73 $\pm$ 1.18   &32.52 $\pm$ 1.17  \\
\hline
\textbf{RRQR}  &35.77  $\pm$ 1.18  &35.18  $\pm$ 1.16   &34.67  $\pm$ 1.16  &34.25  $\pm$ 1.14    \\
\hline
\textbf{Randomfs} &46.26  $\pm$ 1.21  &45.39  $\pm$ 1.19  &44.96  $\pm$ 1.19  &44.65  $\pm$ 1.18  \\
\hline 
\textbf{IG}  &36.24  $\pm$ 1.20  &34.80  $\pm$ 1.19    &33.94  $\pm$ 1.18   &33.39  $\pm$ 1.17 \\
\hline
\end{tabular}
\end{small}
\end{center}
\end{table}
\noindent
We set the value of $r$, i.e. the number of features selected by BSS to 80 and 90 for all experiments. We performed ten-fold cross-validation and repeated it ten times. The value of $\lambda$ was set to 0, 0.1, 0.3, 0.5, 0.7, and 0.9. We compared BSS with RRQR, IG and leverage-score sampling. The mean out-of-sample error was 0 for all methods for both $k=90$ and $k=100$. Table~\ref{tab:synth} shows the set of five most frequently selected features by the different methods for one such synthetic dataset across 100 training sets. The top features picked up by the different methods are the relevant features by construction and also have good discriminatory power. This shows that BSS 
is as good as any other method in terms of feature selection and often picks more discriminatory features than the other methods. We repeated our experiments on ten different synthetic datasets and each time, the five most frequently selected features were from the set of relevant features.
\begin{figure}[!htb]
\centering
\includegraphics[height = 45mm,width= 0.49\columnwidth]{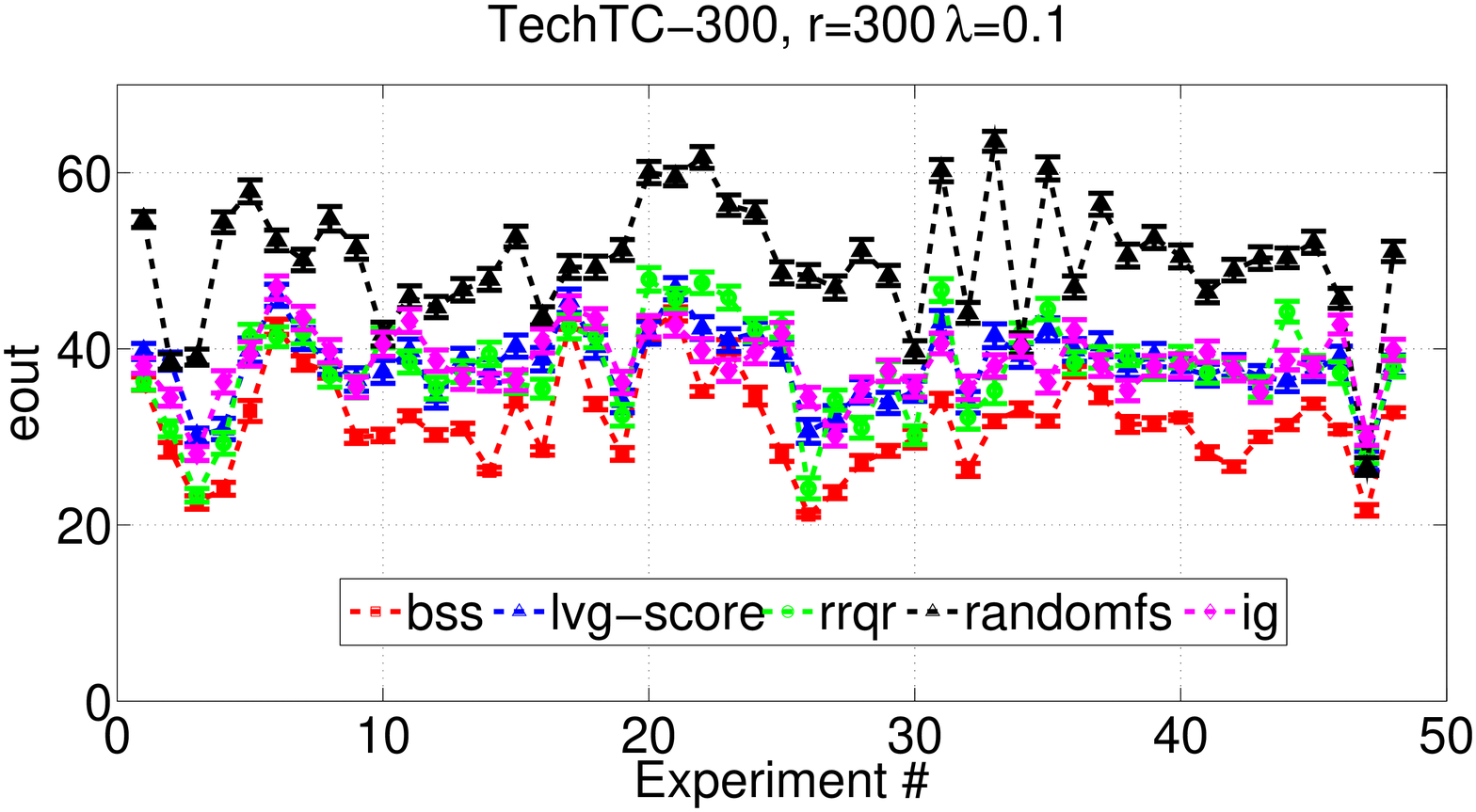}
\includegraphics[height = 45mm,width=0.49\columnwidth]{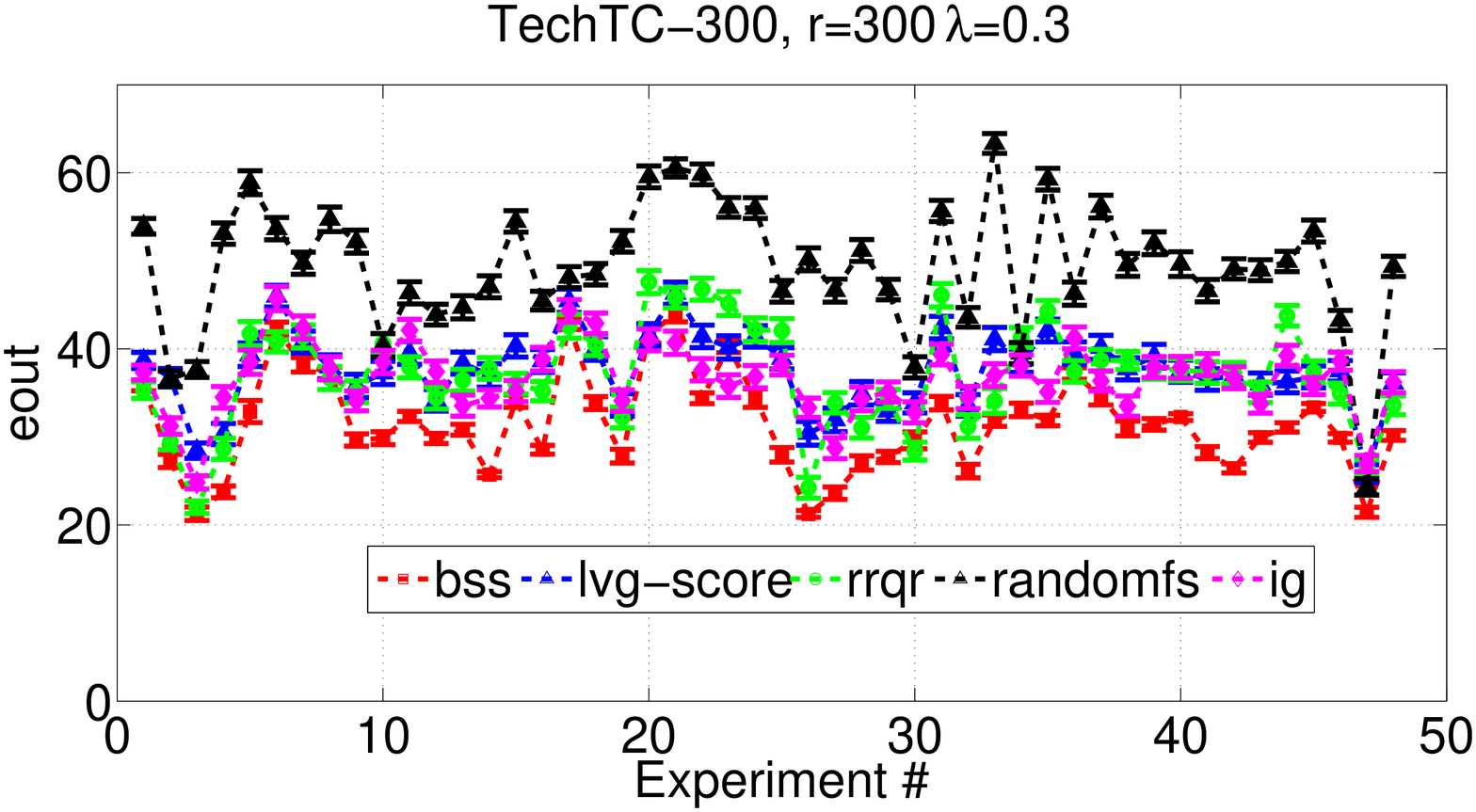}
\includegraphics[height = 45mm,width=0.49\columnwidth]{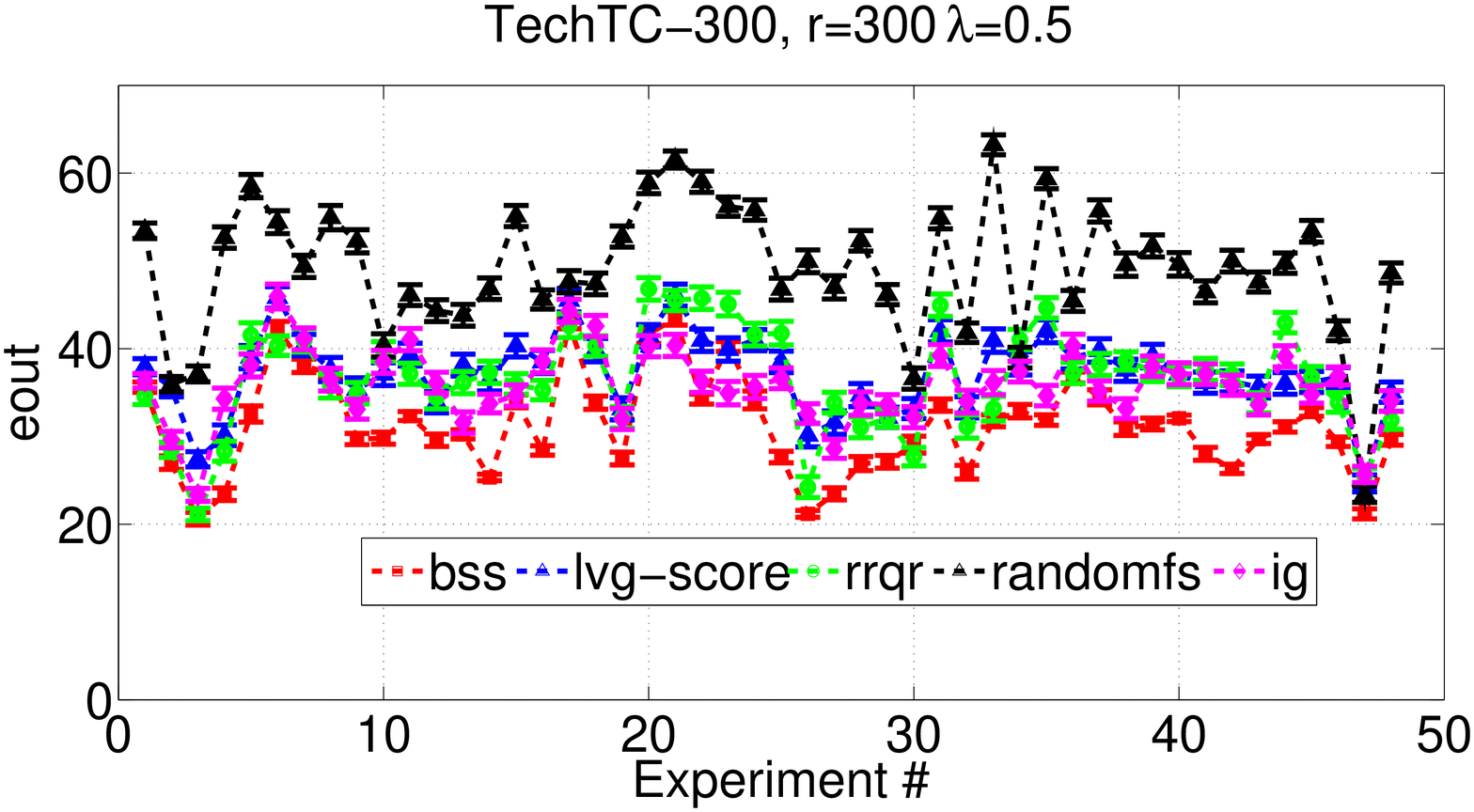}
\includegraphics[height = 45mm,width=0.49\columnwidth]{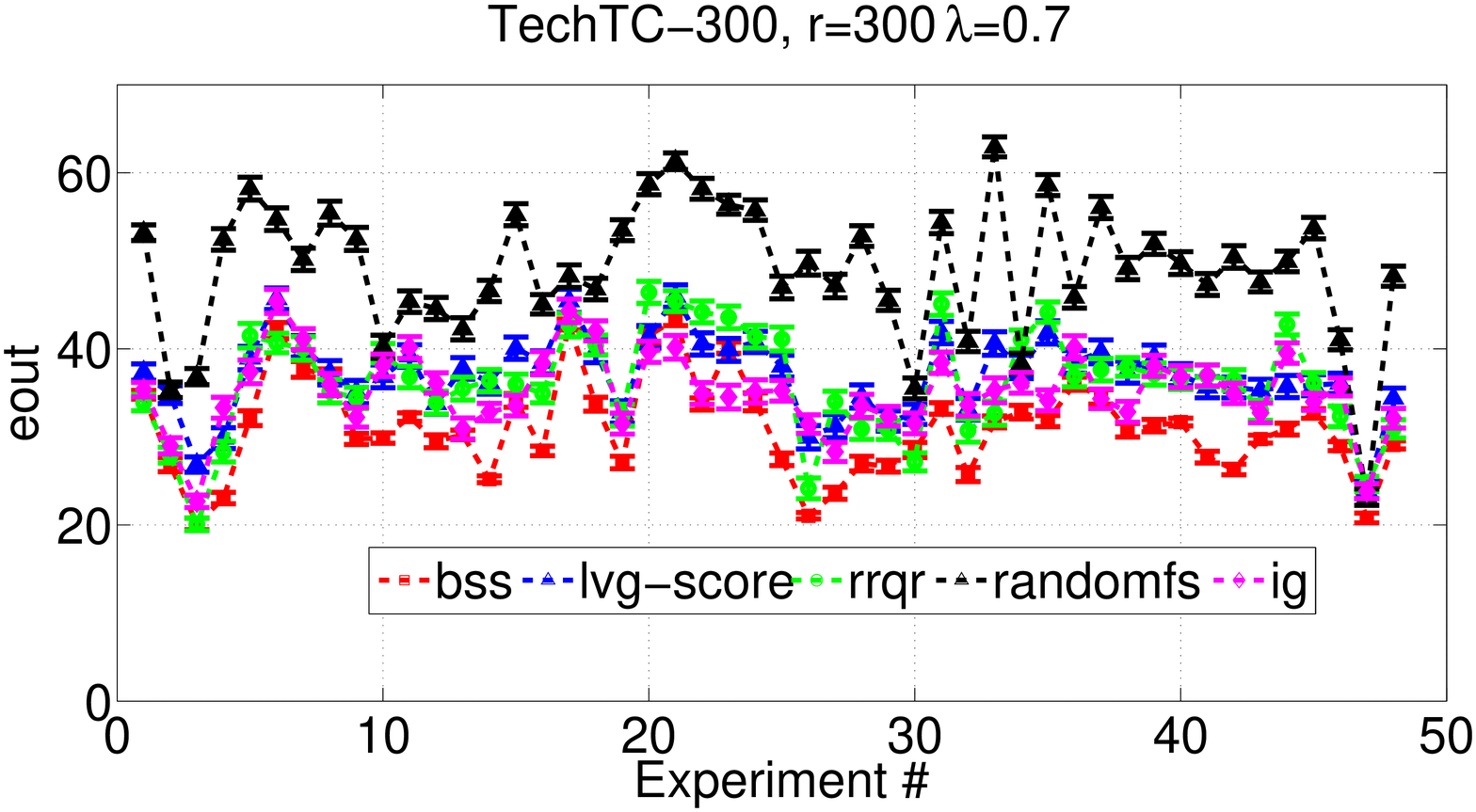} 
\caption{Out-of-sample error of 48 TechTC-300 documents averaged over ten ten-fold cross validation experiments for different values of regularization parameter $\lambda$ and number of features $r=300$. Vertical bars represent standard deviation.}
\label{fig:techtc_eout} 
\end{figure}
Thus, by selecting only 8\%-9\% of all features, we show that we are able to obtain the most discriminatory features along with good out-of-sample error using BSS.

Though running time is not the main subject of this study, we would like to point out that we computed the running time of the different feature selection methods averaged over ten ten-fold cross validation experiments. The time to perform feature selection for each of the methods averaged over ten ten-fold cross-validation experiments was less than a second (See Table~\ref{tab:rlsc_trun}), which shows that the methods can be implemented in practice.

\begin{figure}[!htbp]
\centering
\includegraphics[height = 45mm,width= 0.49\columnwidth]{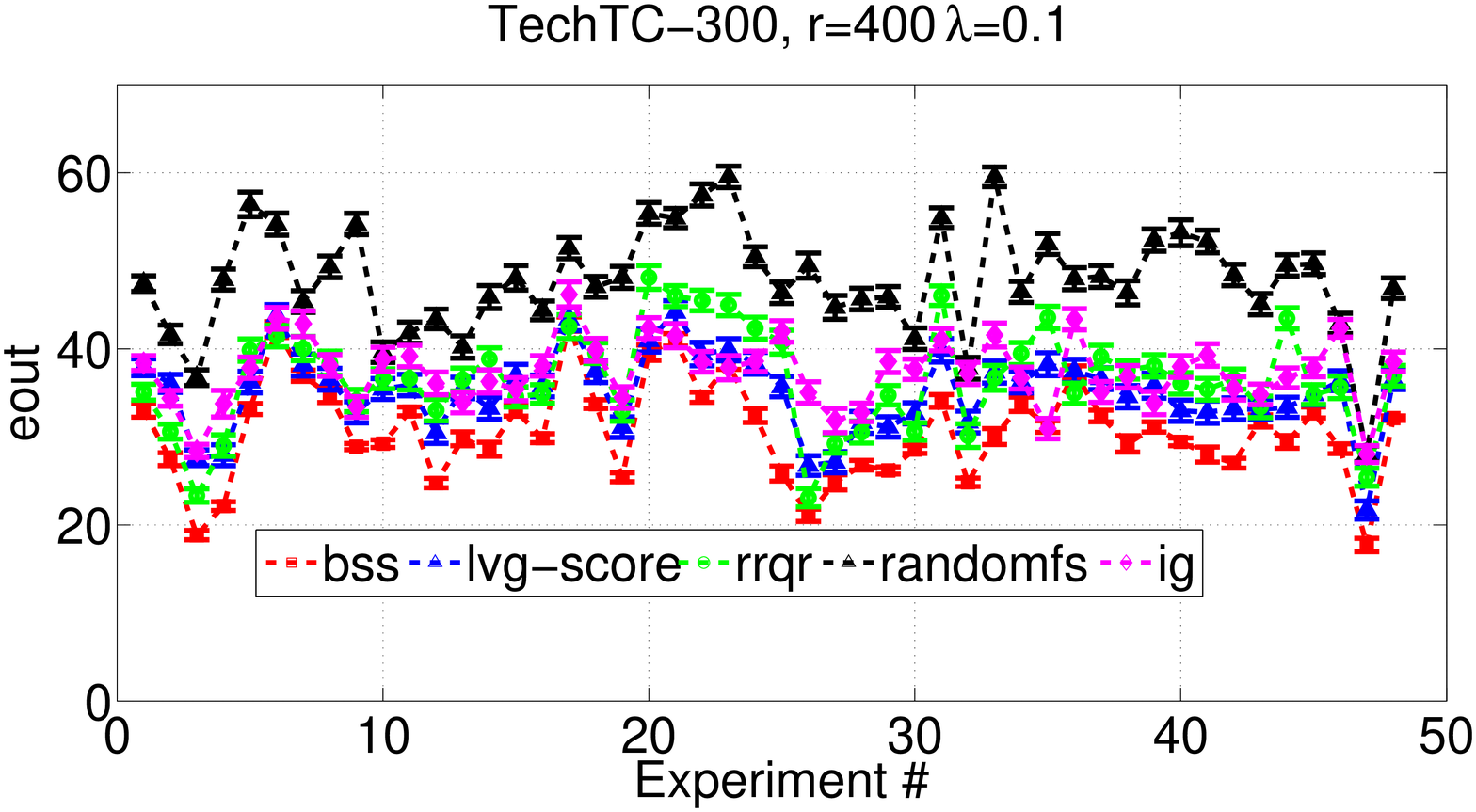}
\includegraphics[height = 45mm,width= 0.49\columnwidth]{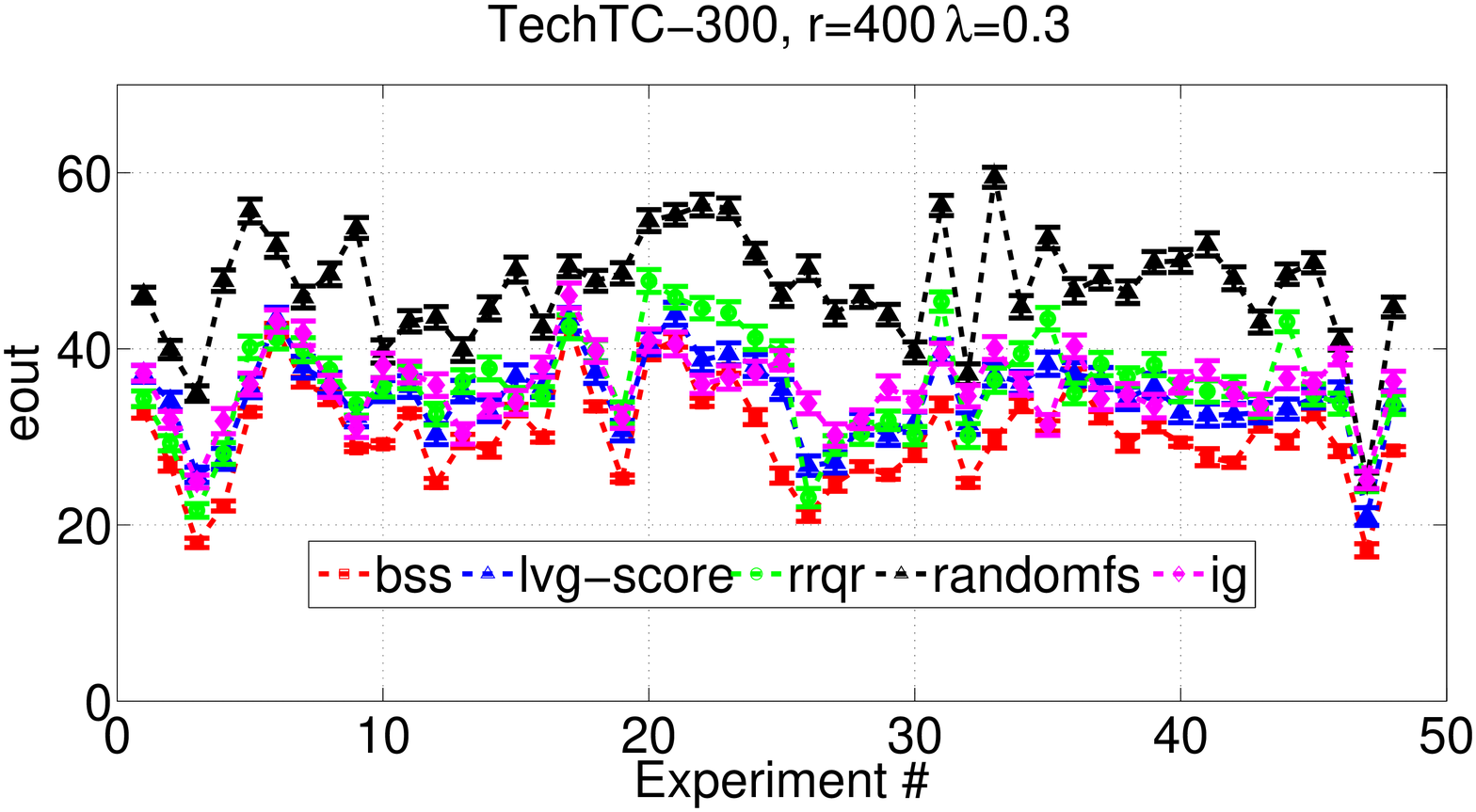}
\includegraphics[height = 45mm,width= 0.49\columnwidth]{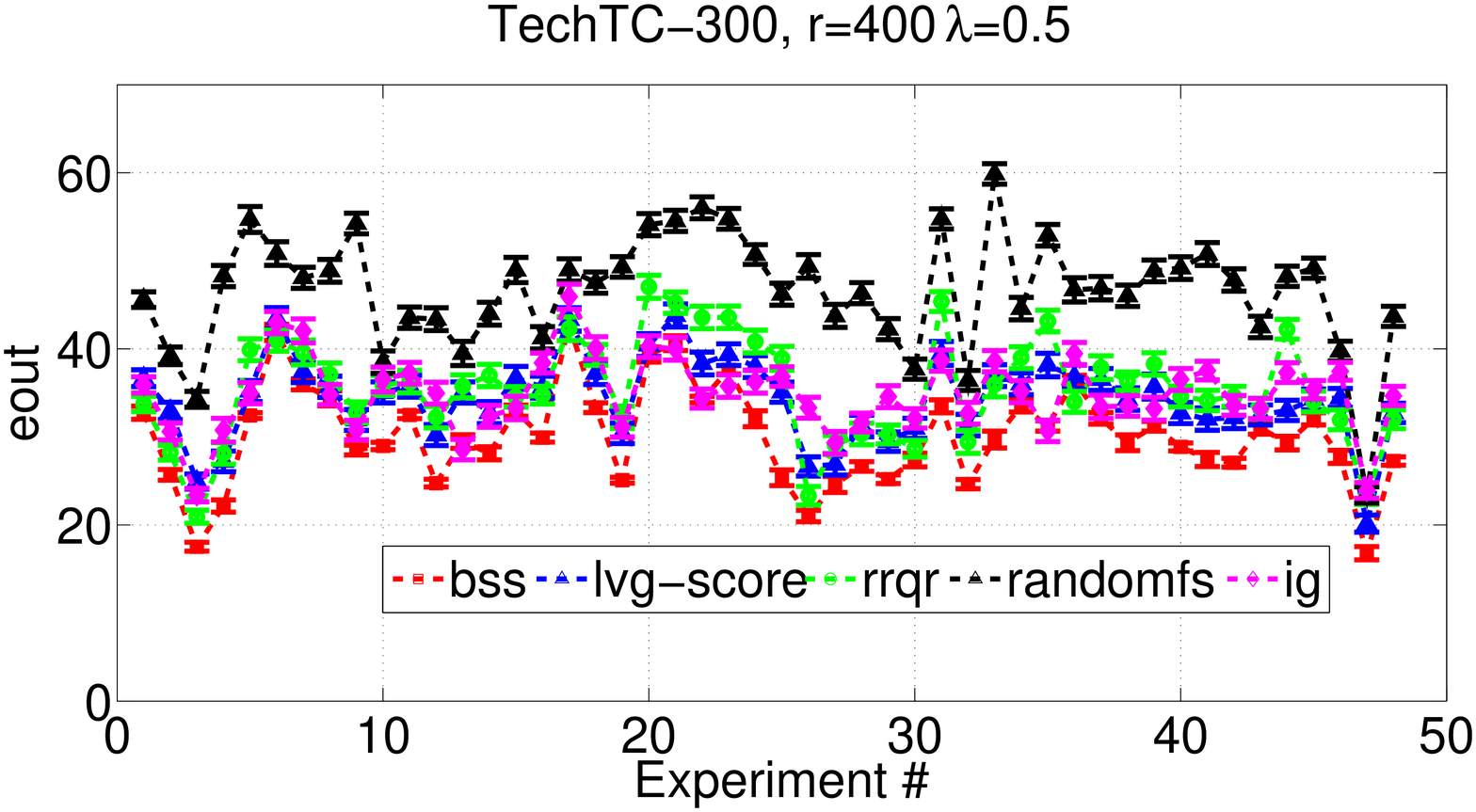}
\includegraphics[height = 45mm,width= 0.49\columnwidth]{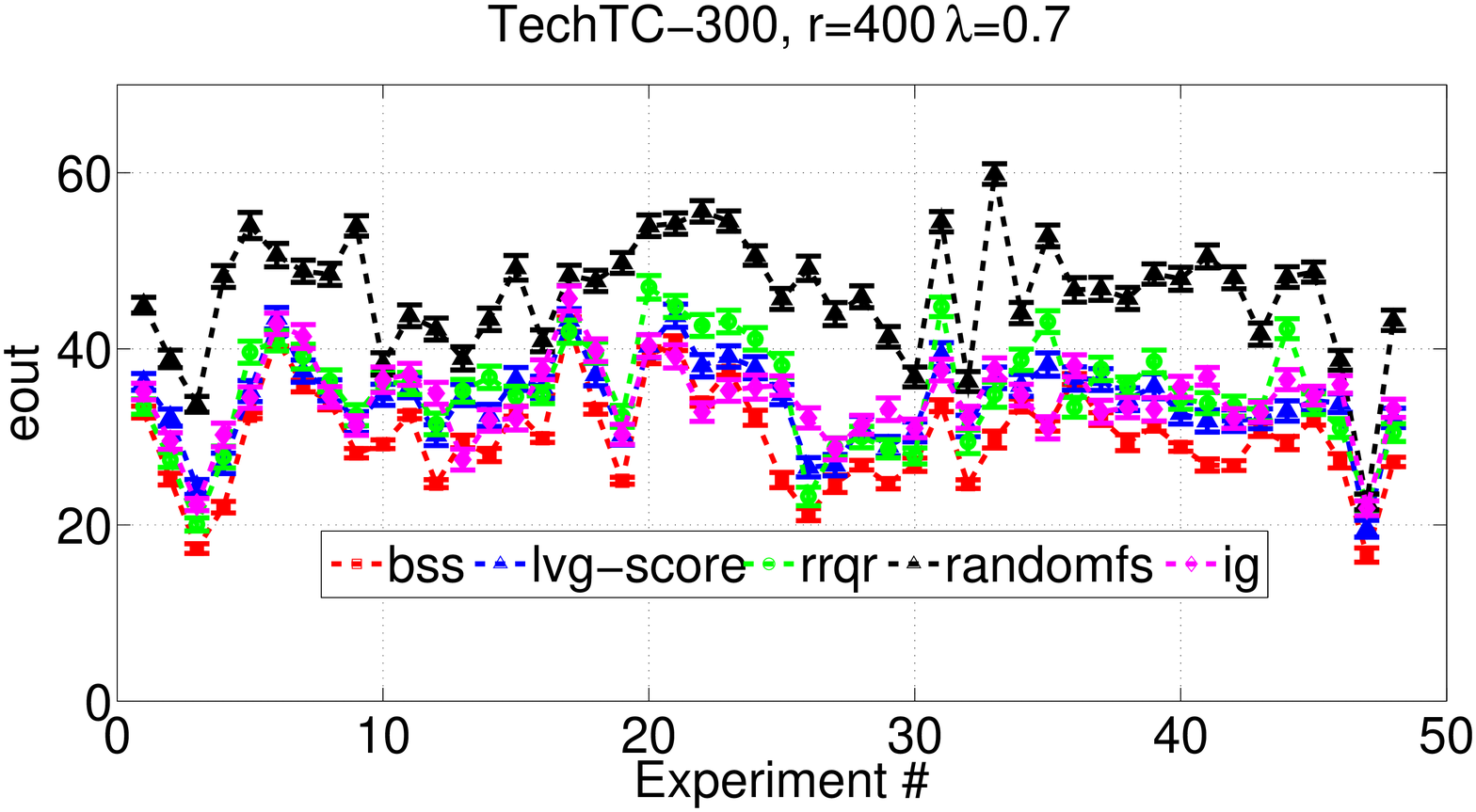}
$r=400$

\includegraphics[height = 45mm,width= 0.49\columnwidth]{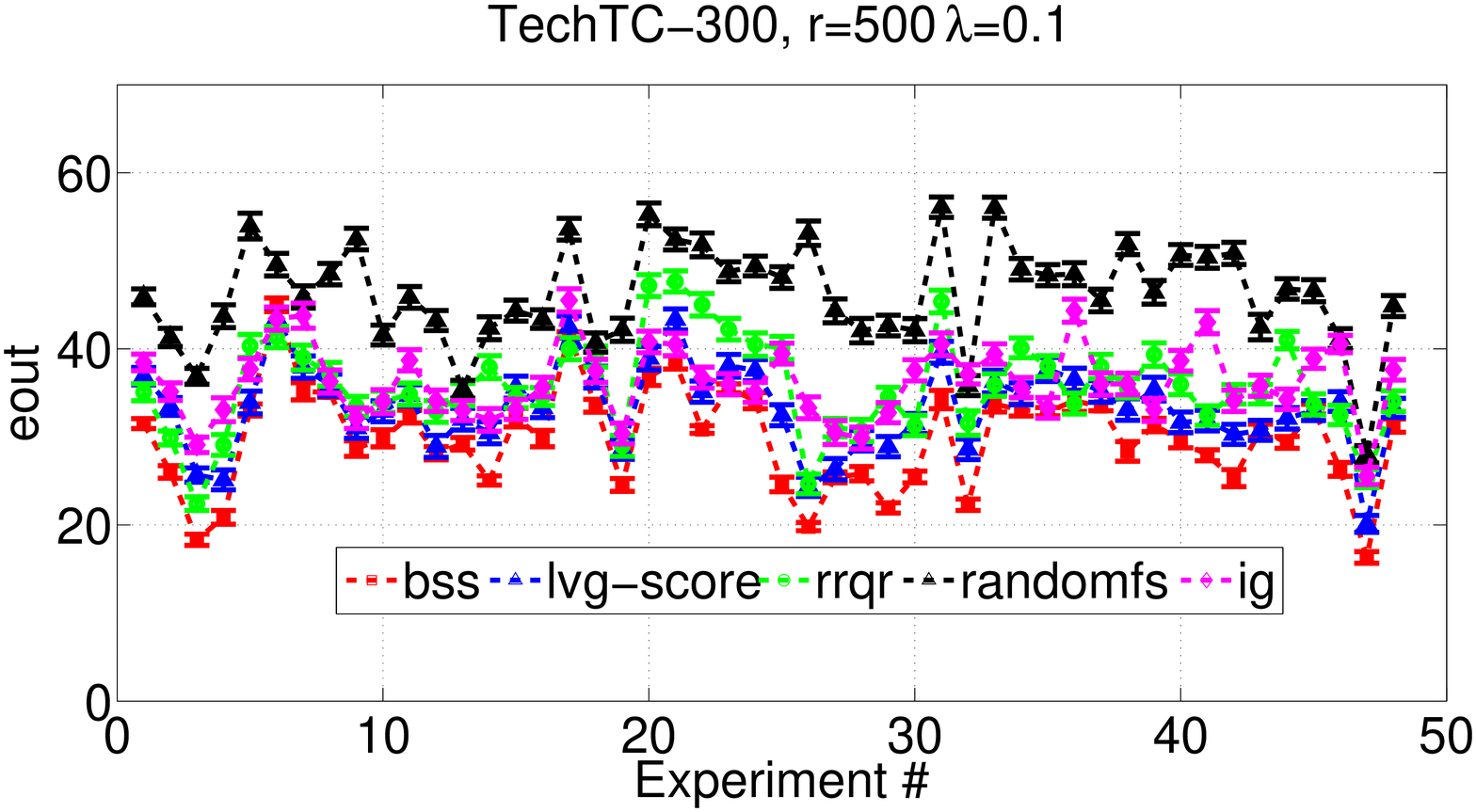}
\includegraphics[height = 45mm,width= 0.49\columnwidth]{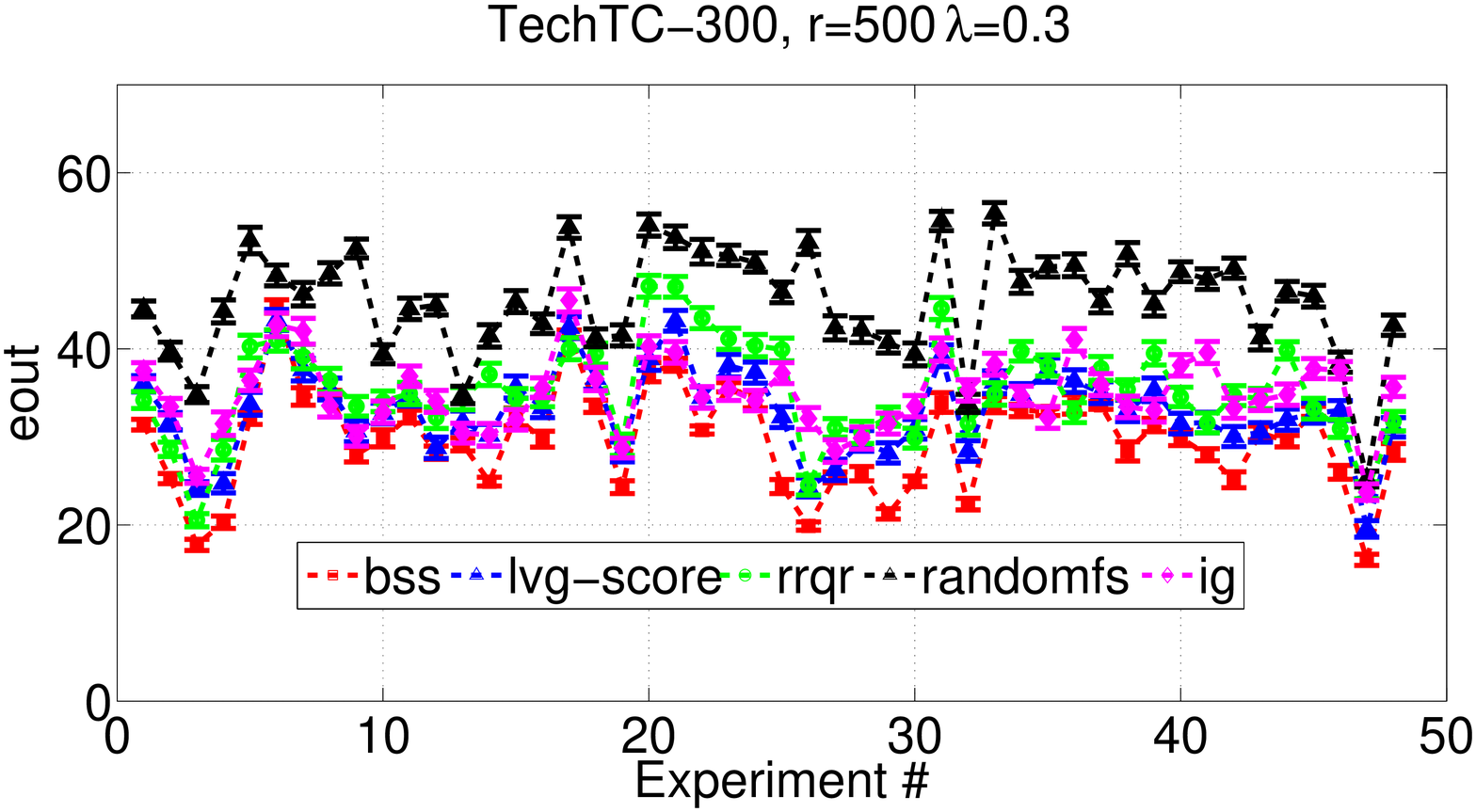}
\includegraphics[height = 45mm,width= 0.49\columnwidth]{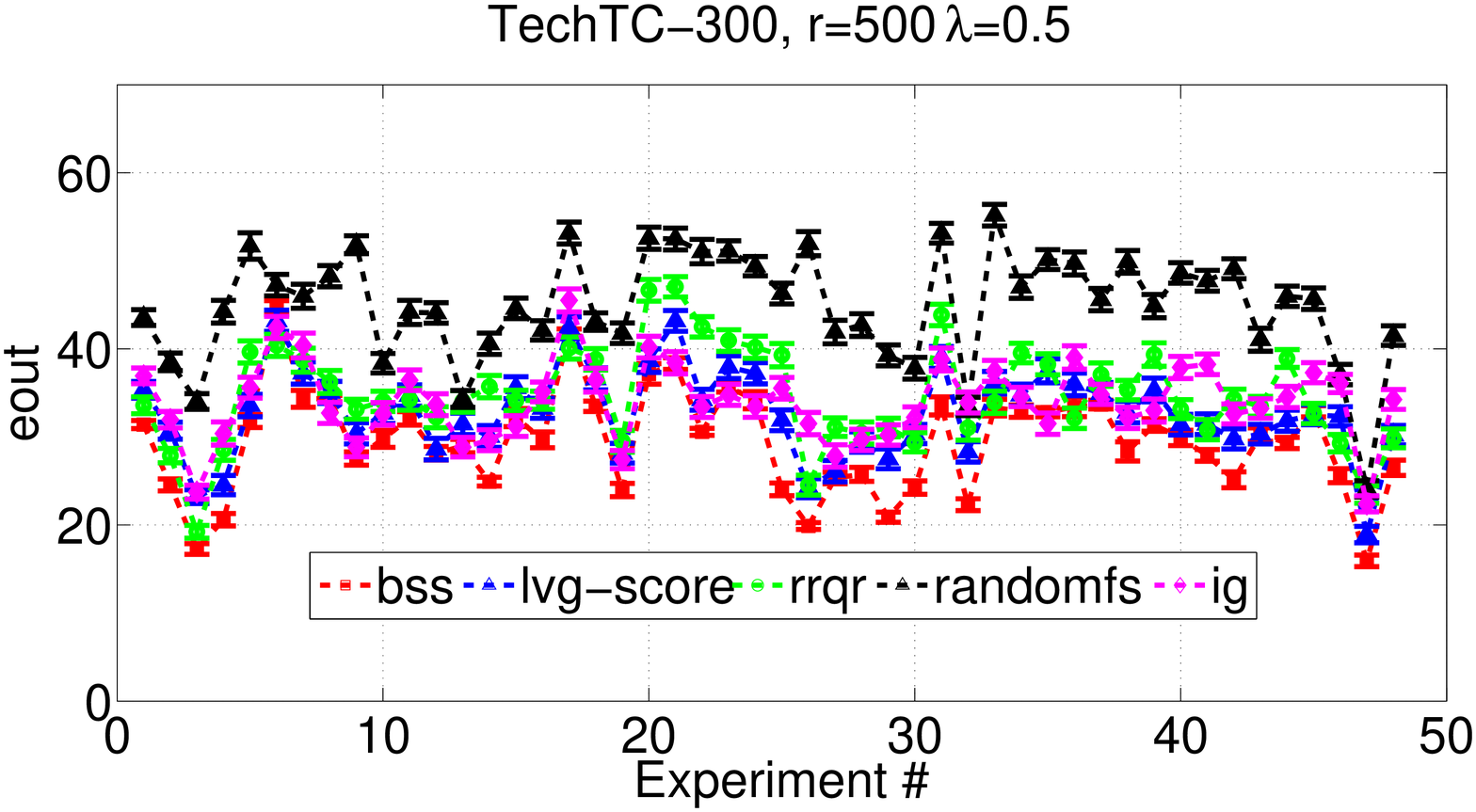}
\includegraphics[height = 45mm,width= 0.49\columnwidth]{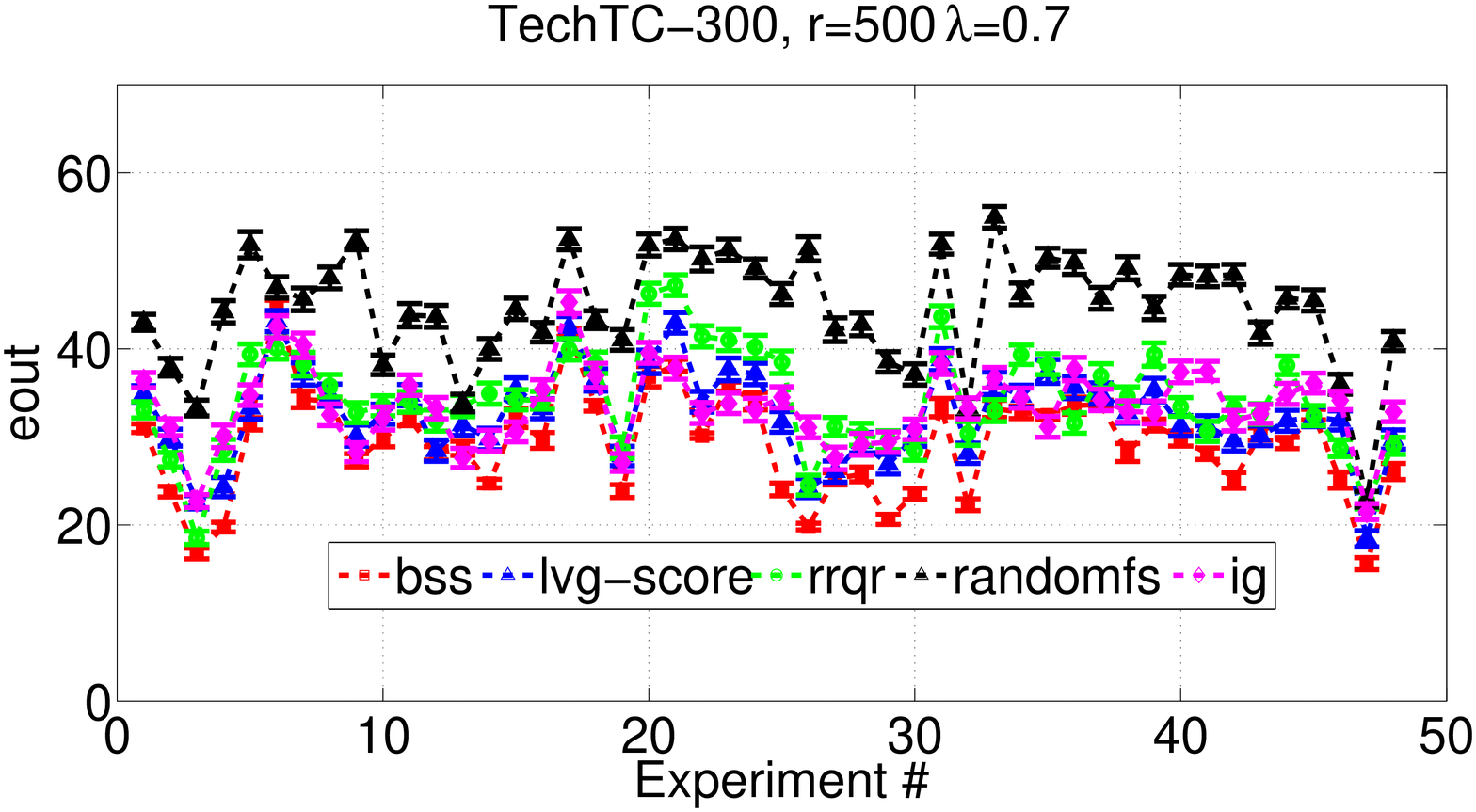}
$r=500$
\caption{Out-of-sample error of 48 TechTC-300 documents averaged over ten ten-fold cross validation experiments for different values of regularization parameter $\lambda$ and number of features $r=400$ and $r=500$. Vertical bars represent standard deviation.}
\label{fig:techtc_eout2}
\end{figure}
\subsubsection{TechTC-300}\label{subsubsec:techttc300}
We use the TechTC-300 data \cite{David04}, consisting of a family of 295 document-term data matrices. The TechTC-300 dataset comes from the Open Directory Project (ODP), which is a large, comprehensive directory of the web, maintained by volunteer editors. Each matrix in the TechTC-300 dataset contains a pair of categories from the ODP. Each category corresponds to a label, and thus the resulting classification task is binary. The documents that are collected from the union of all the subcategories within each category are represented in the bag-of-words model, with the words constituting the features of the data \cite{David04}. Each data matrix consists of 150-280 documents, and each document is described with respect to 10,000-50,000 words. Thus, TechTC-300 provides a diverse collection of data sets for a systematic study of the performance of the RLSC using BSS. We removed all words of length at most four from the datasets. Next we grouped the datasets based on the categories and selected those datasets whose categories appeared at least thrice. There were 147 datasets, and we performed ten-fold cross validation and repeated it ten times on 48 such datasets.  We set the values of the regularization parameter of RLSC to $0.1, 0.3, 0.5$ and $0.7$. 

\begin{table}
\caption{\small A subset of the TechTC matrices of our study.} 
\label{tab:filenames}
\begin{center}
\begin{small}
\begin{tabular}{|c|c|c|}
\hline
 \textbf{id1\_id2} &\textbf{id1}  &\textbf{id2}\\ \hline
1092\_789236  &Arts:Music:Styles:Opera  &US Navy:Decommisioned Submarines \\ \hline
17899\_278949 &US:Michigan:Travel \& Tourism &Recreation:Sailing Clubs:UK \\ \hline
17899\_48446   &US:Michigan:Travel \& Tourism &Chemistry:Analytical:Products \\ \hline
14630\_814096 &US:Colorado:Localities:Boulder &Europe:Ireland:Dublin:Localities \\ \hline
10539\_300332 &US:Indiana:Localities:S  &Canada:Ontario:Localities:E \\ \hline
10567\_11346 &US:Indiana:Evansville  &US:Florida:Metro Areas:Miami \\ \hline
10539\_194915 &US:Indiana:Localities:S &US:Texas:Localities:D \\ \hline
\end{tabular}
\end{small}
\end{center}
\end{table}
\begin{table}
\caption{\small Frequently occurring terms of the TechTC-300 datasets of Table~\ref{tab:filenames} selected by BSS} 
\label{tab:techtc_words}
\begin{center}
\begin{small}
\begin{tabular}{|c|c|}
\hline
\textbf{id1\_id2} &\textbf{words}\\ \hline
1092\_789236 & naval,shipyard,submarine,triton,music,opera,libretto,theatre\\ \hline 
17899\_278949 & sailing,cruising,boat,yacht,racing,michigan,leelanau,casino\\ \hline 
17899\_48446 & vacation,lodging,michigan,asbestos,chemical,analytical,laboratory\\ \hline 
14630\_814096 & ireland,dublin,boulder,colorado,lucan,swords,school,dalkey  \\ \hline
10539\_300332 &ontario,fishing,county,elliot,schererville,shelbyville,indiana,bullet \\ \hline
10567\_11346 &florida,miami,beach,indiana,evansville,music,business,south \\ \hline
10539\_194915 &texas,dallas,plano,denton,indiana,schererville,gallery,north  \\ 
\hline
\end{tabular}
\end{small}
\end{center}
\end{table}
\begin{table}
\caption{\small Frequently occurring terms of the TechTC-300 datasets of Table~\ref{tab:filenames} selected by Leverage-Score Sampling} \label{tab:techtc_words_lvg}
\begin{center}
\begin{small}
\begin{tabular}{|c|c|}
\hline
\textbf{id1\_id2} &\textbf{words}\\ \hline
1092\_789236 & sturgeon, seawolf, skate, triton, frame, opera, finback \\ \hline 
17899\_278949 & sailing, yacht, laser, michigan,breakfast, county, clear\\ \hline 
17899\_48446 & analysis, michigan, water, breakfast, asbestos, environmental, analytical\\ \hline 
14630\_814096 &  ireland, dublin, estate, lucan, dalkey, colorado, boulder \\ \hline
10539\_300332 & library, fishing, service, lodge, ontario, elliot, indiana, shelbyville\\ \hline
10567\_11346 & evansville, services, health, church, south, bullet, florida\\ \hline
10539\_194915 &  dallas, texas, schererville, indiana, shelbyville, plano\\ 
\hline
\end{tabular}
\end{small}
\end{center}
\end{table}
\noindent We set $r$ to 300, 400 and 500. We report the out-of-sample error for all 48 datasets. BSS consistently outperforms Leverage-Score sampling, IG, RRQR and random feature selection on all 48 datasets for all values of the regularization parameter. Table~\ref{tab:techtc_alleout} and Fig~\ref{fig:techtc_eout} shows the results. The out-of-sample error decreases with increase in number of features for all methods. In terms of out-of-sample error, BSS is the best, followed by Leverage-score sampling, IG, RRQR and random feature selection. BSS is at least 3\%-7\% better than the other methods when averaged over 48 document matrices. From Fig~\ref{fig:techtc_eout} and \ref{fig:techtc_eout2}, it is evident that BSS is comparable to the other methods and often better on all 48 datasets. Leverage-score sampling requires greater number of samples to achieve the same out-of-sample error as BSS (See Table~\ref{tab:techtc_alleout}, $r=500$ for Lvg-Score and $r=300$ for BSS). Therefore, for the same number of samples, BSS outperforms leverage-score sampling in terms of out-of-sample error. The out-of-sample error of supervised IG is worse than that of unsupervised BSS, which could be due to the worse generalization of the supervised IG metric. We also observe that the out-of-sample error decreases with increase in $\lambda$ for the different feature selection methods.\\
\noindent
We list the most frequently occurring words selected by BSS and leverage-score sampling for the $r=300$ case for seven TechTC-300 datasets over 100 training sets used in the cross-validation experiments. Table~\ref{tab:filenames} shows the names of the seven TechTC-300 document-term matrices. The words shown in Tables ~\ref{tab:techtc_words} and ~\ref{tab:techtc_words_lvg} were selected in all cross-validation experiments for these seven datasets. The words are closely related to the categories to which the documents belong, which shows that BSS and leverage-score sampling select important features from the training set. For example, for the document-pair $(1092\_789236)$, where $1092$ belongs to the category of ``Arts:Music:Styles:Opera" and $789236$ belongs to the category of ``US:Navy: Decommisioned Submarines", the BSS algorithm selects submarine, shipyard, triton, opera, libretto, theatre which are closely related to the two classes. The top words selected by leverage-score sampling for the same document-pair are seawolf, sturgeon, opera, triton finback, which are closely related to the class. Another example is the document-pair $10539\_300332$, where $10539$ belongs to ``US:Indiana:Localities:S" and $300332$ belongs to the category of ``Canada: Ontario: Localities:E". The top words selected for this document-pair are ontario, elliot, shelbyville, indiana, schererville which are closely related to the class values. Thus, we see that using only 2\%-4\% of all features we are able to select relevant features and obtain good out-of-sample error. The top words selected by leverage-score sampling are library, fishing, elliot, indiana, shelbyville, ontario which are closely related to the class.\\
Though feature selection is an offline task, we give a discussion of the running times of the different methods to highlight that BSS can be implemented in practice. We computed the running time of the different feature selection methods averaged over ten ten-fold cross validation experiments and over 48 datasets (See Table~\ref{tab:rlsc_trun}). The average time for feature selection by BSS is approximately over a minute, while the rest of the methods take less than a second.
This shows that BSS can be implemented in practice and can scale up to reasonably large datasets with 20,000-50,000 features. For BSS and leverage-score sampling, the running time includes the compute to compute SVD of the matrix. BSS takes approximately a minute to select features, but is at least 3\%-7\% better in terms of out-of-sample error than the other methods. IG takes less than a second to select features, but is 4\%-7\% worse than BSS in terms of out-of-sample error. 

\subsection{Experiments on Ridge Regression in the fixed design setting}
In this section, we describe experiments on feature selection on ridge regression in the fixed design setting using synthetic and real data.
\begin{figure}[!htb]
\centering
\includegraphics[height = 45mm,width= 0.49\columnwidth]{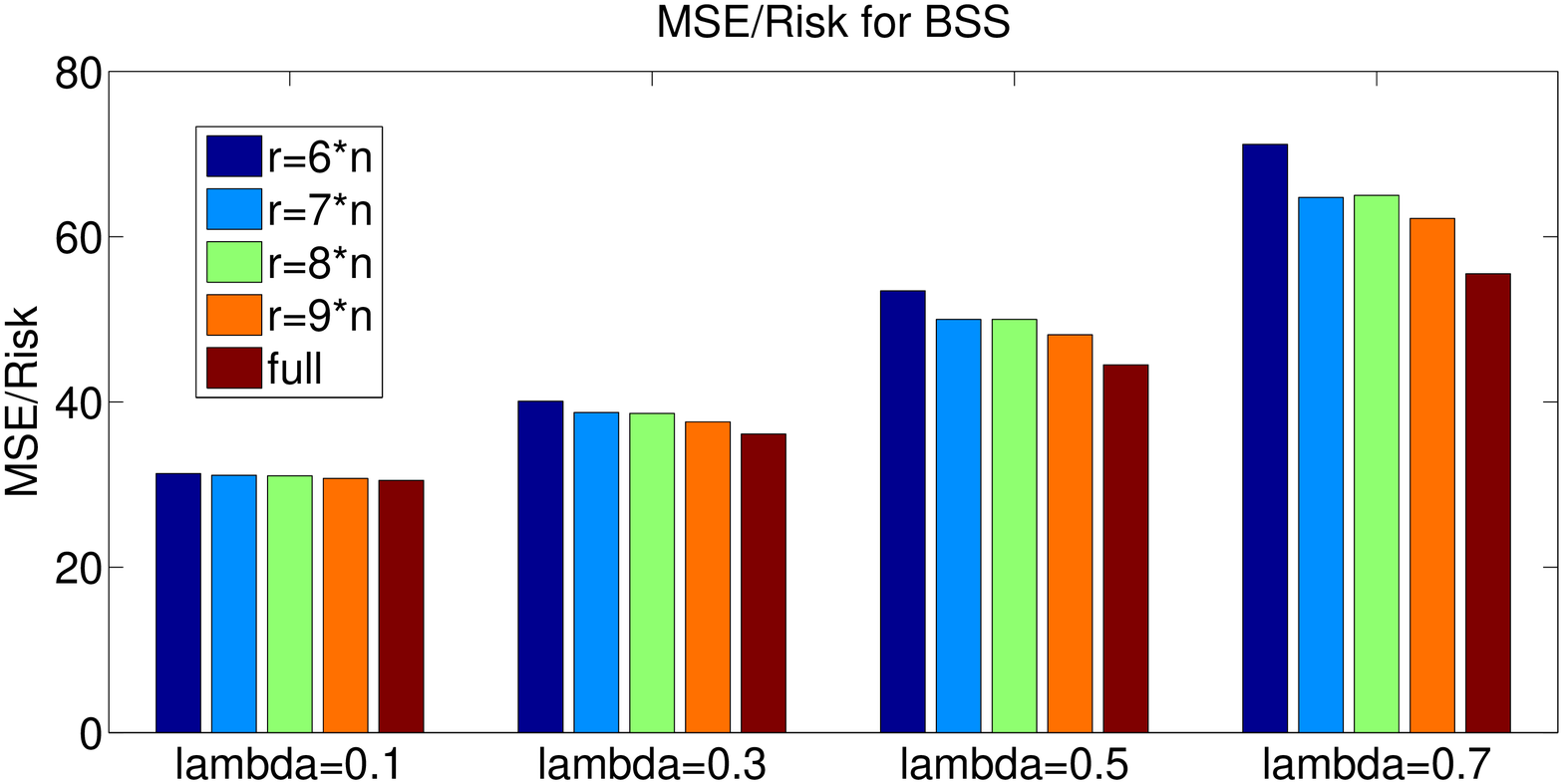}
\includegraphics[height = 45mm,width= 0.49\columnwidth]{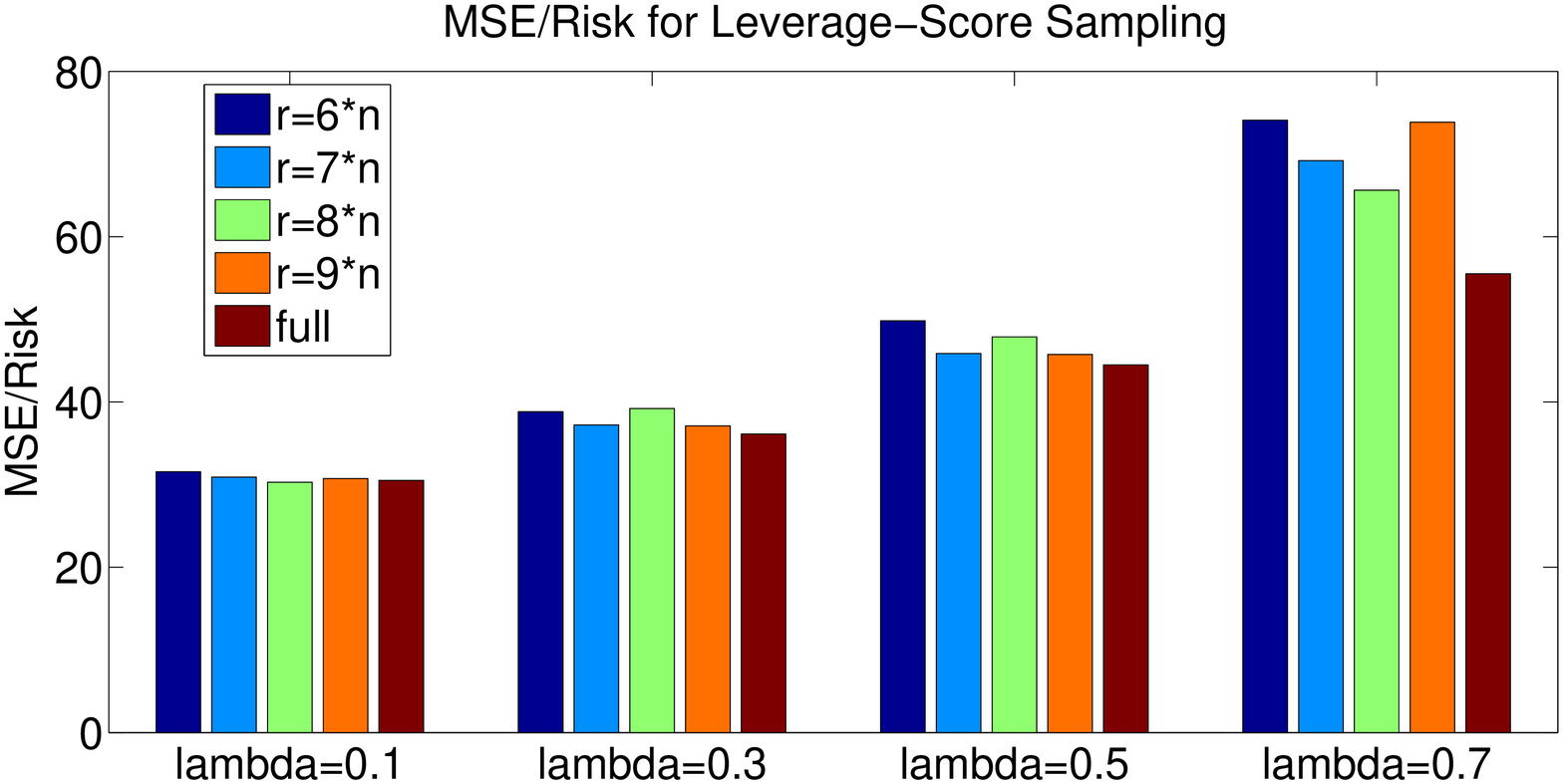}
$k=90$\\
\includegraphics[height = 45mm,width= 0.49\columnwidth]{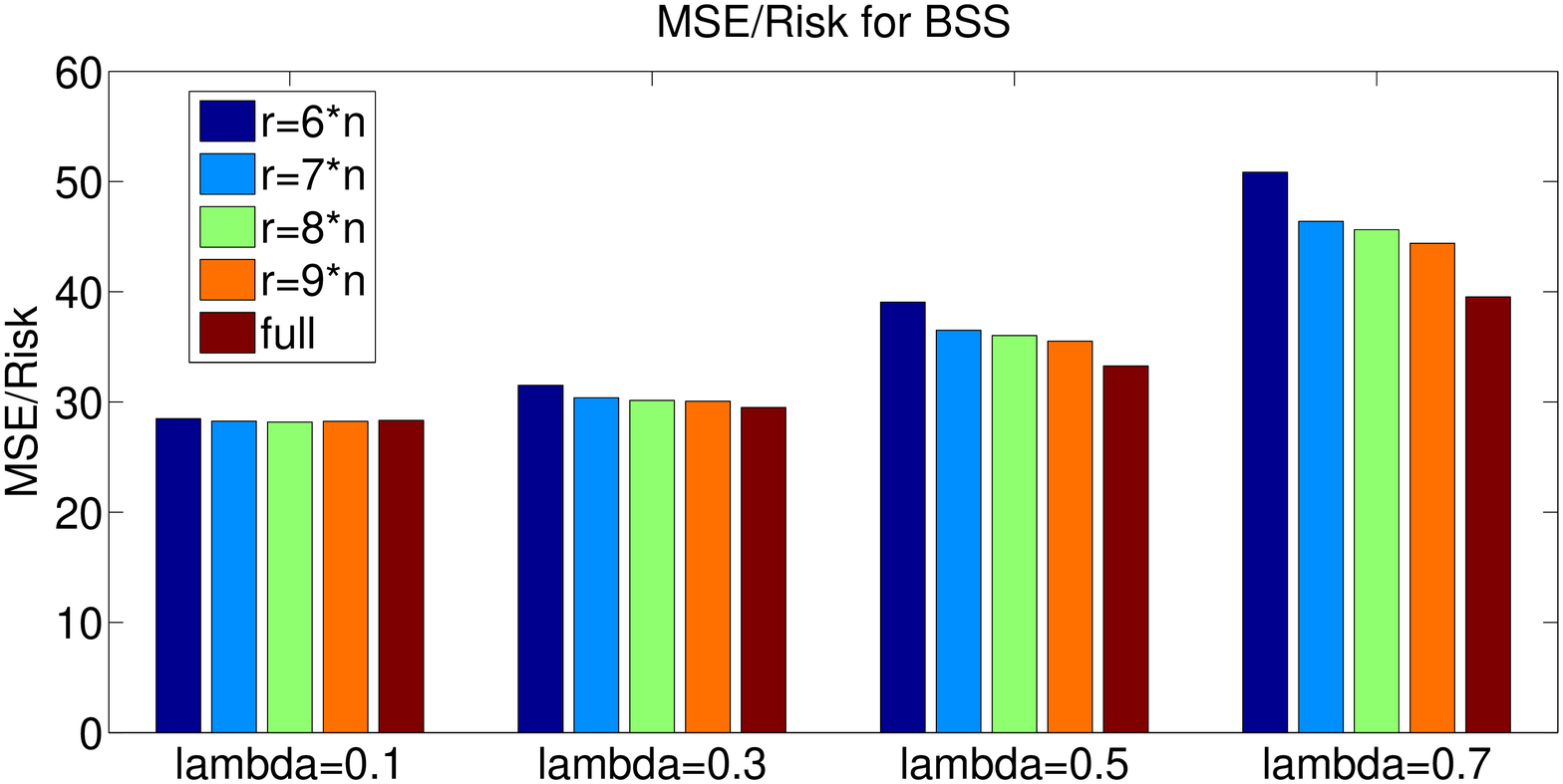}
\includegraphics[height = 45mm,width= 0.49\columnwidth]{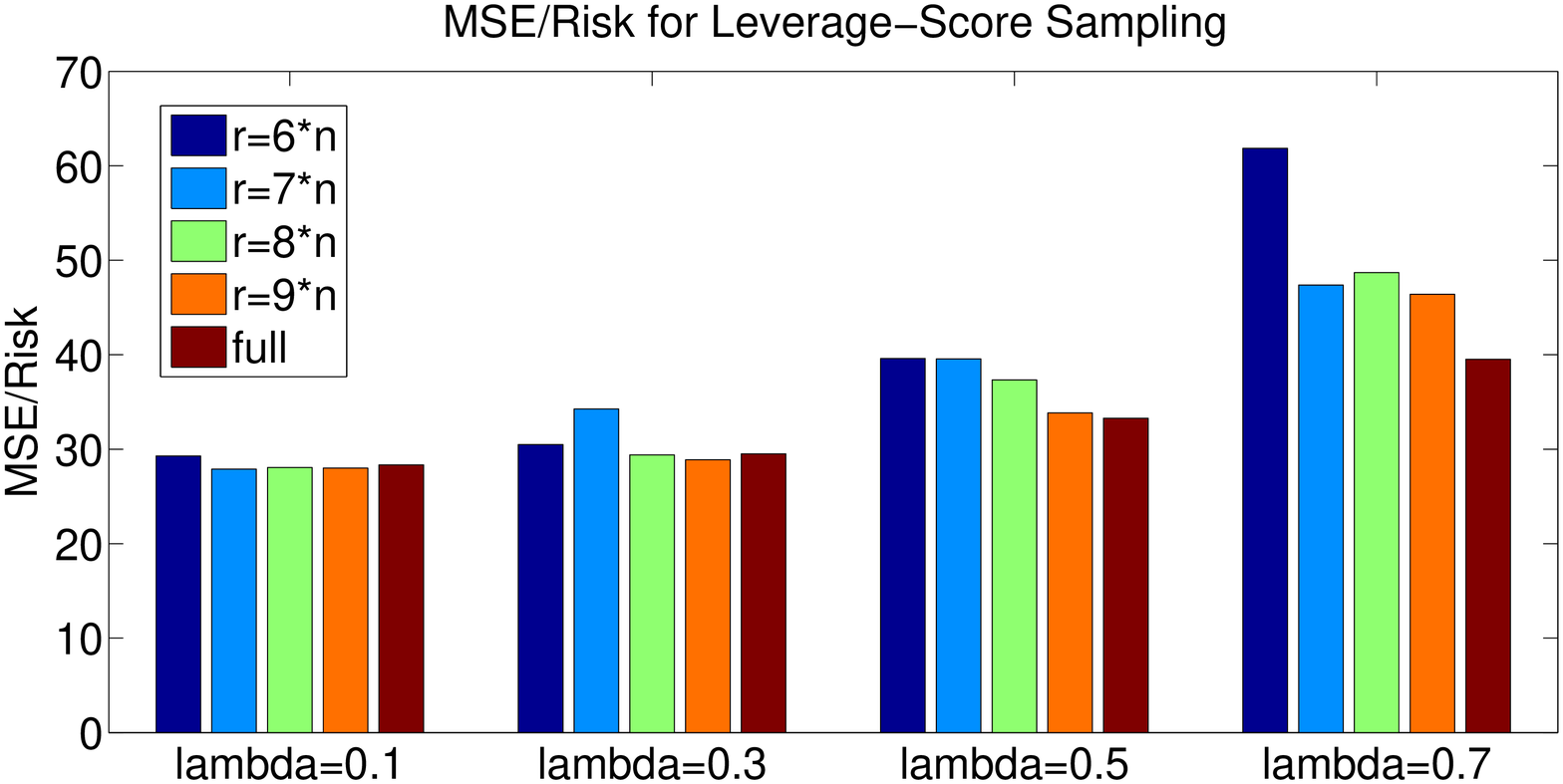}
$k=100$
\caption{MSE/Risk for synthetic data for $k=90$ and $k=100$ using different feature selection methods as a function of $\lambda$. The risk after feature selection is comparable to the risk of full-data.}
\label{fig:risk_bss}
\end{figure}
\subsubsection{Synthetic Data}
We generate the features of the synthetic data $\matX$ in the same manner as described in Section~\ref{subsubsec:synth}. We generate $\betavec \sim \mathcal{N}(0,1)$ and $\y =\matX^T \betavec + \omegavec$, where $\omegavec \in \mathbb{R}^n$ and $\betavec \in \mathbb{R}^d.$
We set $n$ to 30 and $d$ to 1000. We set the number of relevant features, $k$ to 90 and 100 and ran two sets of experiments. We set the value of $r$, i.e. the number of features selected by BSS and leverage-score sampling to $t*n$, where $t=6, 7, 8, 9$ for both experiments. The value of $\lambda$ was set to 0.1, 0.3, 0.5 and 0.7. We compared the risk of ridge regression using BSS and leverage-score sampling with the risk of full-feature selection and report the MSE/Risk in the fixed design setting as a measure of accuracy. Fig~\ref{fig:risk_bss} shows the risk of synthetic data for both BSS and leverage-score sampling as a function of $\lambda$. The risk of the sampled data is comparable to the risk of the full-data in most cases, which follows from our theory. We observe that for higher values of $\lambda$, the risk of sampled space becomes worse than that of full-data for both BSS and leverage-score sampling. The risk in the sampled space is almost the same for both BSS and Leverage-score sampling. The time to compute feature selection is less than a second for both methods (See Table~\ref{tab:rr_trun}).
\begin{figure}[!htb]
\centering
\includegraphics[height = 45mm,width= 0.49\columnwidth]{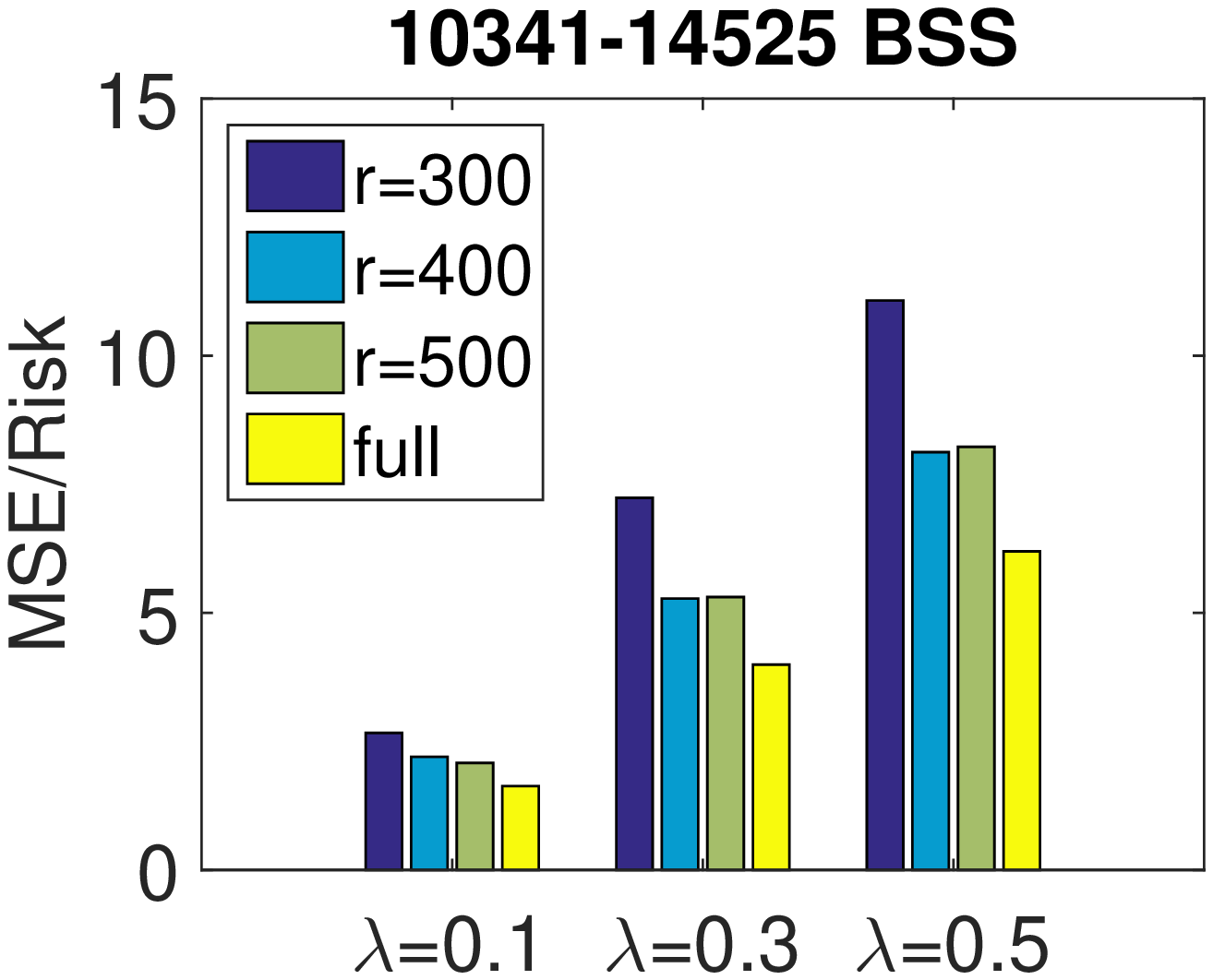}
\includegraphics[height = 45mm,width= 0.49\columnwidth]{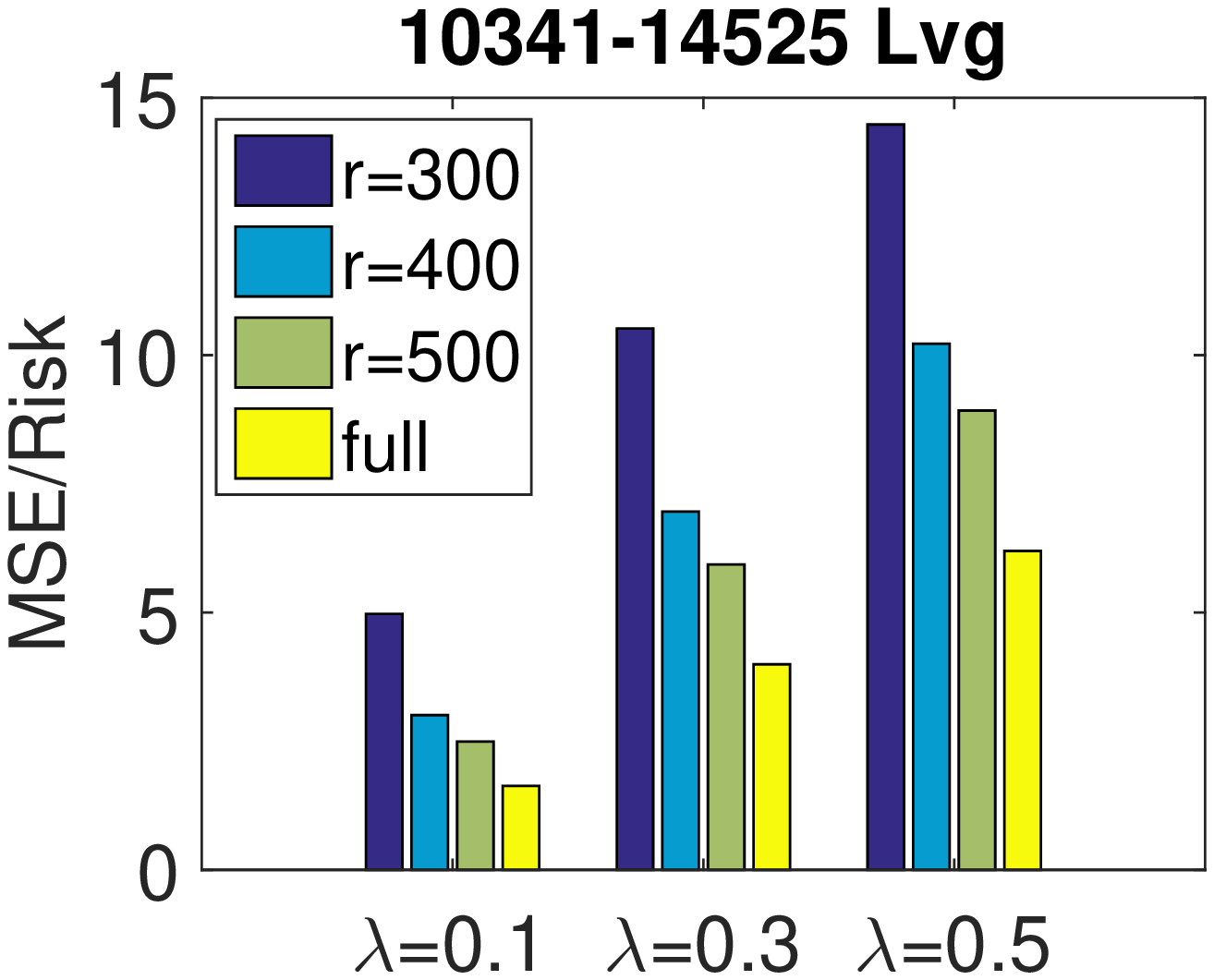}
\includegraphics[height = 45mm,width= 0.49\columnwidth]{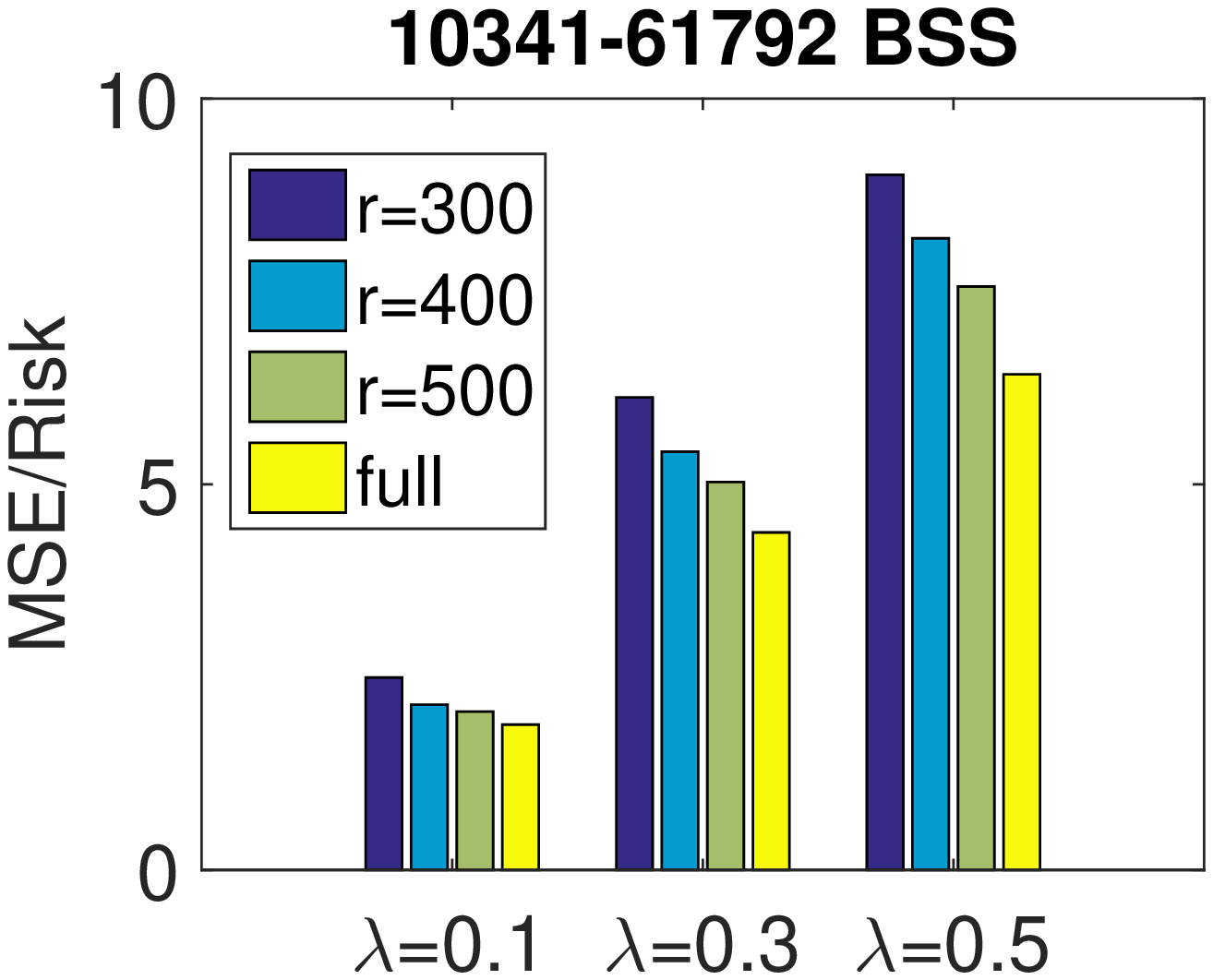}
\includegraphics[height = 45mm,width= 0.49\columnwidth]{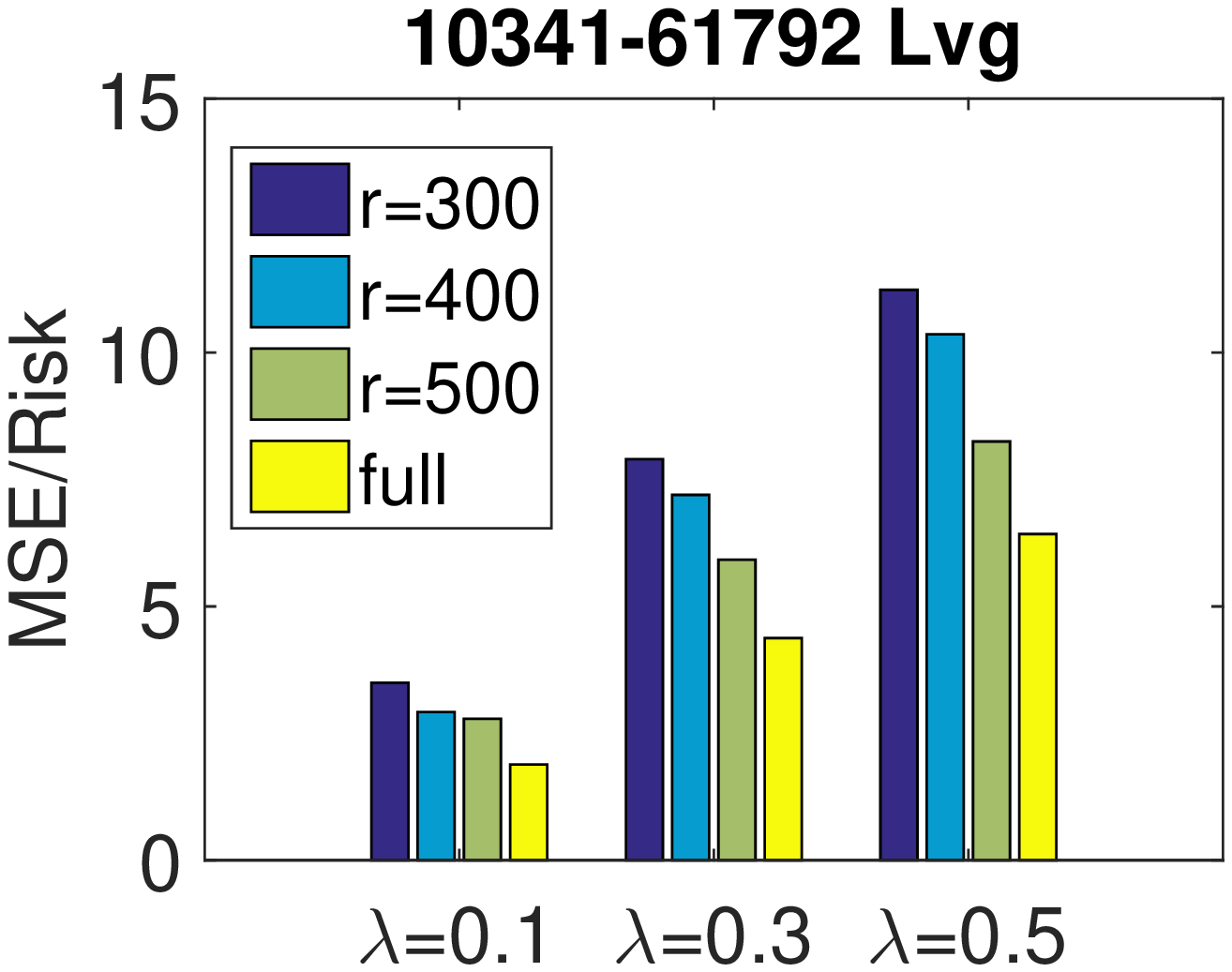}
\caption{MSE/Risk for TechTC-300 data using different feature selection methods as a function of $\lambda$. The risk after feature selection is comparable to the risk of full-data.}
\label{fig:risk_techtc}
\end{figure}

\begin{table}[!htb]
\caption{\small Running time of various feature selection methods in seconds. For synthetic data, the running time corresponds to the experiment when $r=8n.$ For TechTC-300, the running time corresponds to the experiment when $r=400.$}
\label{tab:rr_trun}
\begin{center}
\begin{small}
\begin{tabular}{|c||c|c|c|c|}
\hline
 &Synthetic Data  &TechTC (10341-14525) &TechTC (10341-61792)  \\ \hline
BSS &0.3368 &68.8474 &67.013  \\ \hline
LVG &0.0045 &0.3994 &0.3909 \\ \hline
\end{tabular}
\end{small}
\end{center}
\end{table}

\subsubsection{TechTC-300}
We use two TechTC-300 datasets, namely ``10341-14525" and ``10341-61792" to illustrate our theory. 
We add gaussian noise to the labels. We set the value of $r,$ the number of features to be selected to 300, 400 and 500. The value of $\lambda$ was set to 0.1, 0.3 and 0.5. We compared the risk of ridge regression using BSS and leverage-score sampling with the risk of full-feature selection and report the MSE/Risk in the fixed design setting as a measure of accuracy. Fig~\ref{fig:risk_techtc} shows the risk of real data for both BSS and leverage-score sampling as a function of $\lambda$. The risk of the sampled data is comparable to the risk of the full-data in most cases, which follows from our theory. The risk of the sampled data decreases with increase in $r.$ The time to perform feature selection is approximately a minute for BSS and less than a second for leverage-score sampling (See Table~\ref{tab:rr_trun}).

\section{Conclusion}
We present a provably accurate feature selection method for RLSC which works well empirically and also gives better generalization peformance than prior existing methods. The number of features required by BSS is of the order $O\left(n/\epsilon^2\right)$, which makes the result tighter than that obtained by leverage-score sampling. BSS has been recently used as a feature selection technique for k-means clustering \citep{Bouts13}, linear SVMs \citep{Paul14} and our work on RLSC helps to expand research in this direction. The risk of ridge regression in the sampled space is comparable to the risk of ridge regression in the full feature space in the fixed design setting and we observe this in both theory and experiments. An interesting future work in this direction would be to include feature selection for non-linear kernels with provable guarantees.

\noindent \textbf{Acknowledgements.}
Most of the work was done when SP was a graduate student at RPI.
This work is supported by NSF CCF 1016501 and NSF IIS 1319280.

\bibliographystyle{apalike}
\bibliography{references}

\begin{thebibliography}{}

\bibitem[Agarwal, 2002]{Agarwal02}
Agarwal, D. (2002).
\newblock Shrinkage estimator generalizations of proximal support vector
  machines.
\newblock In {\em Proceedings of the eighth ACM SIGKDD international conference
  on Knowledge discovery and data mining}, pages 173--182.

\bibitem[Avron et~al., 2013]{Avron13}
Avron, H., Sindhwani, V., and Woodruff, D. (2013).
\newblock Sketching structured matrices for faster nonlinear regression.
\newblock In {\em Advances in Neural Information Processing Systems}, pages
  2994--3002.

\bibitem[Bach, 2013]{Bach13}
Bach, F. (2013).
\newblock Sharp analysis of low-rank kernel matrix approximations.
\newblock In {\em The 26th Annual Conference on Learning Theory, (COLT)}, pages
  185--209.

\bibitem[Batson et~al., 2009]{BSS09}
Batson, J., Spielman, D., and Srivastava, N. (2009).
\newblock Twice-ramanujan sparsifiers.
\newblock In {\em Proceedings of the 41st annual ACM STOC}, pages 255--262.

\bibitem[Bhattacharyya, 2004]{Bhat04}
Bhattacharyya, C. (2004).
\newblock Second order cone programming formulations for feature selection.
\newblock {\em JMLR}, 5:1417--1433.

\bibitem[Boutsidis and Magdon-Ismail, 2013]{Bouts13}
Boutsidis, C. and Magdon-Ismail, M. (2013).
\newblock Deterministic feature selection for $ k $-means clustering.
\newblock {\em IEEE Transactions on Information Theory}, 59(9):6099-- 6110.

\bibitem[Dasgupta et~al., 2007]{Dasgup07}
Dasgupta, A., Drineas, P., Harb, B., Josifovski, V., and Mahoney, M. (2007).
\newblock Feature selection methods for text classification.
\newblock In {\em Proceedings of the 13th ACM SIGKDD International Conference
  on Knowledge Discovery and Data Mining}, pages 230--239.

\bibitem[Davidov et~al., 2004]{David04}
Davidov, D., Gabrilovich, E., and Markovitch, S. (2004).
\newblock Parameterized generation of labeled datasets for text categorization
  based on a hierarchical directory.
\newblock In {\em Proceedings of the 27th Annual International ACM SIGIR
  Conference}, pages 250--257.
\newblock \url{http://techtc.cs.technion.ac.il/techtc300/techtc300.html}.

\bibitem[Demmel and Veselic, 1992]{demmel}
Demmel, J. and Veselic, K. (1992).
\newblock Jacobi's method is more accurate than qr.
\newblock {\em SIAM Journal on Matrix Analysis and Applications},
  13(4):1204--1245.

\bibitem[Drineas et~al., 2006]{Drineas06}
Drineas, P., Mahoney, M., and Muthukrishnan, S. (2006).
\newblock Sampling algorithms for l2 regression and applications.
\newblock In {\em Proceedings of the 17th Annual ACM-SIAM SODA}, pages
  1127--1136.

\bibitem[Fung and Mangasarian, 2001]{fung01}
Fung, G. and Mangasarian, O. (2001).
\newblock Proximal support vector machine classifiers.
\newblock In {\em Proceedings of the seventh ACM SIGKDD international
  conference on Knowledge discovery and data mining}, pages 77--86.

\bibitem[Lu et~al., 2013]{Dhillon13}
Lu, Y., Dhillon, P., Foster, D., and Ungar, L. (2013).
\newblock Faster ridge regression via the subsampled randomized hadamard
  transform.
\newblock In {\em Advances in Neural Information Processing Systems 26}, pages
  369--377.

\bibitem[Paul et~al., 2015]{Paul14}
Paul, S., Magdon{-}Ismail, M., and Drineas, P. (2015).
\newblock Feature selection for linear {SVM} with provable guarantees.
\newblock In {\em Proceedings of the Eighteenth International Conference on
  Artificial Intelligence and Statistics, {(AISTATS)}}, pages 735--743.

\bibitem[Poggio and Smale, 2003]{Poggio03}
Poggio, T. and Smale, S. (2003).
\newblock The mathematics of learning: Dealing with data.
\newblock {\em Notices of the AMS}, 50(5):537--544.

\bibitem[Rifkin et~al., 2003]{RifkinRLSC}
Rifkin, R., Yeo, G., and Poggio, T. (2003).
\newblock Regularized least-squares classification.
\newblock {\em Nato Science Series Sub Series III Computer and Systems
  Sciences}, 190:131--154.

\bibitem[Rudelson and Vershynin, 2007]{Rudelson}
Rudelson, M. and Vershynin, R. (2007).
\newblock Sampling from large matrices: An approach through geometric
  functional analysis.
\newblock {\em J. ACM}, 54(4).

\bibitem[Stewart and Sun, 1990]{stewart}
Stewart, G. and Sun, J. (1990).
\newblock Matrix perturbation theory.

\bibitem[Suykens and Vandewalle, 1999]{Suykens99}
Suykens, J. and Vandewalle, J. (1999).
\newblock Least squares support vector machine classifiers.
\newblock {\em Neural processing letters}, 9(3):293--300.

\bibitem[Yang and Pedersen, 1997]{yang97}
Yang, Y. and Pedersen, J. (1997).
\newblock A comparative study on feature selection in text categorization.
\newblock In {\em ICML}, volume~97, pages 412--420.

\bibitem[Zhang and Peng, 2004]{PZhang04}
Zhang, P. and Peng, J. (2004).
\newblock {SVM} vs regularized least squares classification.
\newblock In {\em Proceedings of the 17th International Conference on Pattern
  Recognition}, volume~1, pages 176--179.

\bibitem[Zhang and Oles, 2001]{TZhang01}
Zhang, T. and Oles, F. (2001).
\newblock Text categorization based on regularized linear classification
  methods.
\newblock {\em Information retrieval}, 4(1):5--31.

\end{thebibliography}

\end{document}